\documentclass[11pt]{article}
\date{}
\usepackage[utf8]{inputenc} 
\usepackage[T1]{fontenc}    
\usepackage[colorlinks=true,linkcolor=blue,allcolors=blue]{hyperref}       
\usepackage{url}            
\usepackage{booktabs}       
\usepackage{amsfonts}       
\usepackage{nicefrac}       
\usepackage{microtype,nicefrac}      
\usepackage{xspace}
\usepackage{natbib}
\usepackage[toc]{appendix}
\usepackage{minitoc}
\usepackage{subcaption}
\usepackage{physics}
\usepackage{comment}
\usepackage{amsmath,amsthm}
\DeclareMathOperator*{\argmin}{argmin}
\DeclareMathOperator*{\argmax}{argmax}

\usepackage{graphicx}
\usepackage{rotating}

\usepackage{amssymb}
\usepackage{autobreak}
\usepackage{mathtools}
\usepackage[dvipsnames]{xcolor}
\usepackage{bbm}
\usepackage[ruled,vlined, linesnumbered]{algorithm2e}
\usepackage{algorithmic}
\usepackage[parfill]{parskip}
\usepackage{thmtools}
\usepackage[toc]{appendix}
\usepackage{minitoc}

\makeatletter

\title{Online Min-Max Optimization: From Individual Regrets\\to Cumulative Saddle Points}

\author{%
Abhijeet Vyas\\
  Purdue University\\
  \texttt{vyas26@purdue.edu} 
\and
Brian Bullins\\
  Purdue University\\
  \texttt{bbullins@purdue.edu} \\
}

\begin{document}

\newcommand{\diag}{\operatorname{diag}}
\newcommand{\innp}[1]{\left\langle #1 \right\rangle}
\newcommand{\bdot}[1]{\mathbf{\dot{ #1 }}}
\newcommand{\OPT}{\operatorname{OPT}}
\newcommand{\mA}{\mathbf{A}}
\newcommand{\ones}{\mathds{1}}
\newcommand{\zeros}{\textbf{0}}
\newcommand{\vx}{\mathbf{x}}
\newcommand{\tF}{\Tilde{F}}
\newcommand{\cF}{\mathcal{F}}
\newcommand{\vp}{\mathbf{p}}
\newcommand{\cx}{\mathcal{X}}
\newcommand{\cy}{\mathcal{Y}}
\newcommand{\cu}{\mathcal{U}}
\newcommand{\cf}{\mathcal{F}}
\newcommand{\cfb}{\bar{\mathcal{F}}}
\newcommand{\cc}{\mathcal{C}}
\newcommand{\cl}{\mathcal{K}}
\newcommand{\cz}{\mathcal{Z}}
\newcommand{\vxh}{\mathbf{\hat{x}}}
\newcommand{\vyh}{\mathbf{\hat{y}}}
\newcommand{\vzh}{\mathbf{\hat{z}}}
\newcommand{\vy}{\mathbf{y}}
\newcommand{\vz}{\mathbf{z}}
\newcommand{\vv}{\mathbf{v}}
\newcommand{\ve}{\mathbf{e}}
\newcommand{\vw}{\mathbf{w}}
\newcommand{\vvh}{\mathbf{\hat{v}}}
\newcommand{\vg}{\mathbf{g}}
\newcommand{\vub}{\bar{\mathbf{u}}}
\newcommand{\vvb}{\bar{\mathbf{v}}}
\newcommand{\vuh}{\hat{\mathbf{u}}}
\newcommand{\veta}{\bm{\eta}}
\newcommand{\vetab}{\bar{\bm{\eta}}}
\newcommand{\vetah}{\bm{\hat{\eta}}}
\newcommand{\defeq}{\stackrel{\mathrm{\scriptscriptstyle def}}{=}}
\newcommand{\etal}{\textit{et al}.}
\newcommand{\tnabla}{\widetilde{\nabla}}
\newcommand{\tE}{\widetilde{E}}
\newcommand{\rr}{\mathbb{R}}
\newcommand{\ee}{\mathbb{E}}

\newcommand{\littlesum}{\mathop{\textstyle\sum}}

\newcommand{\adgt}{\textsc{adgt}}

\makeatletter
\def\mathcolor#1#{\@mathcolor{#1}}
\def\@mathcolor#1#2#3{%
  \protect\leavevmode
  \begingroup
    \color#1{#2}#3%
  \endgroup
}
\makeatother

\newcommand*{\sepfbox}[1]{%
  \begingroup
    \sbox0{\fbox{#1}}%
    \setlength{\fboxrule}{0pt}%
    \fbox{\unhbox0}%
  \endgroup
}
\newcommand*{\vsepfbox}[1]{%
  \begingroup
    \sbox0{\fbox{#1}}%
    \setlength{\fboxrule}{0pt}%
    \mbox{\kern-\fboxsep\fbox{\unhbox0}\kern-\fboxsep}%
  \endgroup
}

\newtheorem{theorem}{Theorem}[section]
\numberwithin{theorem}{section}
\newtheorem{corollary}[theorem]{Corollary}
\newtheorem{conjecture}{Conjecture}
\newtheorem{lemma}[theorem]{Lemma}
\newtheorem{proposition}[theorem]{Proposition}
\newtheorem{claim}[theorem]{Claim}
\newtheorem{fact}[theorem]{Fact}
 \newtheorem{definition}[theorem]{Definition}
\newtheorem{finalremark}[theorem]{Final Remark}
 \newtheorem{remark}[theorem]{Remark}
 \newtheorem{example}[theorem]{Example}
\newtheorem{observation}[theorem]{Observation}
\newtheorem{assumption}{Assumption}

\makeatletter
\newcommand{\subalign}[1]{%
  \vcenter{%
    \Let@ \restore@math@cr \default@tag
    \baselineskip\fontdimen10 \scriptfont\tw@
    \advance\baselineskip\fontdimen12 \scriptfont\tw@
    \lineskip\thr@@\fontdimen8 \scriptfont\thr@@
    \lineskiplimit\lineskip
    \ialign{\hfil$\m@th\scriptstyle##$&$\m@th\scriptstyle{}##$\hfil\crcr
      #1\crcr
    }%
  }%
}
\makeatother

\newcommand{\sne}{$\text{SNE-Reg}_T$}
\newcommand{\sdg}{$\text{SDual-Gap}_T$}
\newcommand{\ec}{$\text{min-max EC}$}
\newcommand{\ommns}{\text{OMMNS}}
\newcommand{\dyne}{$\text{DNE-Reg}_T$}
\newcommand{\dgap}{$\text{Dual-Gap}_T$}
\newcommand{\dsp}{$\text{DSP-Reg}_T$}
\newcommand{\dreg}{$\text{D-Regret}_T$}
\newcommand{\sareg}{$\text{SA-Regret}_T$}
\newcommand{\sasp}{$\text{SASP-Regret}_T$}
\newcommand{\sreg}{$\text{S-Regret}_T$}

\newcommand\xx{\boldsymbol{\mathit{x}}}
\newcommand\yy{\boldsymbol{\mathit{y}}}
\newcommand\zz{\boldsymbol{\mathit{z}}}
\newcommand{\R}{\mathbb{R}}

\newcommand\calB{\mathcal{B}}
\newcommand\calX{\mathcal{X}}
\newcommand\calY{\mathcal{Y}}
\newcommand\calZ{\mathcal{Z}}

\newcommand{\pa}[1]{\left(#1\right)}
\newcommand{\ang}[1]{\left\langle #1\right\rangle}

\doparttoc 
\faketableofcontents 

\maketitle

\begin{abstract}
We propose and study an online version of min-max optimization based on cumulative saddle points under a variety of performance measures beyond convex-concave settings. After first observing the incompatibility of (static) Nash equilibrium (SNE-Reg$_T$) with individual regrets even for strongly convex-strongly concave functions, we propose an alternate \emph{static} duality gap (\sdg) inspired by the online convex optimization (OCO) framework. We provide algorithms that, using a reduction to classic OCO problems, achieve bounds for \sdg~and a novel \emph{dynamic} saddle point regret (\dsp), which we suggest naturally represents a min-max version of the dynamic regret in OCO. We derive our bounds for \sdg~and \dsp~under strong convexity-strong concavity and a min-max notion of exponential concavity (min-max EC), and in addition we establish a class of functions satisfying \ec~that captures a two-player variant of the classic portfolio selection problem. Finally, for a dynamic notion of regret compatible with individual regrets, we derive bounds under a two-sided Polyak-\L{}ojasiewicz (PL) condition.
\end{abstract}
\section{Introduction}

Online convex optimization (OCO) is a foundational framework that has been extensively studied over the past several decades for sequential decision-making tasks \citep{zinkevich2003online, hazan2016introduction}. In this setting, we consider the static regret after $T$ rounds, i.e.,
\begin{equation}\label{eqn:sreg}
    \text{S-Regret}_T = \sum_{t=1}^T g_t(z_t) - \min_{z\in \cz} \sum_{t=1}^T g_t(z)
\end{equation}
where $\{g_t\}_{t=1}^T$ is a sequence of convex functions for a convex decision set $\mathcal{Z}$ and $\{z_t\}_{t=1}^T$ are the actions of the player. $\text{S-Regret}_T$ has served as a standard performance measure, with well-established tight bounds of $\Theta(\sqrt{T})$ for general convex functions and $\Theta(\log(T))$ for strongly convex functions~\citep{abernethy2008optimal}. Meanwhile, for \emph{dynamic} regret, which accounts for changing environments, numerous algorithms have been developed for convex \citep{besbes2015non, cesa2012mirror, zhang2018adaptive, hsieh2021adaptive} and strongly convex functions \citep{mokhtari2016online, zhang2018dynamic,zhao2021improved}. 

While the OCO framework generalizes the offline minimization problem $\min_{z\in \cz} f(z)$, the min-max optimization problem, relevant to a multitude of fields such as adversarial machine learning, multi-agent reinforcement learning, robust optimization and risk assessment is defined as
$$\min_{x\in \cx}\max_{y\in \cy} f(x,y)$$ for convex decision sets $\cx \subseteq \mathbb{R}^m$, $\cy \subseteq \mathbb{R}^n$, and where we let $d := m+n$ denote the overall problem dimension. Building on the OCO framework and min-max optimization, the emerging field of \emph{learning in time-varying zero-sum games} has gained significant attention, particularly in addressing bilinear and convex-concave settings. The key objective of this setting is that each player aims to minimize their individual regret, i.e.,
\begin{align*}
    \text{Reg}^1_T = \sum_{t=1}^T f_t(x_t,y_t)-\min_x \sum_{t=1}^T f_t(x,y_t),~~~~~\text{Reg}^2_T = \max_y \sum_{t=1}^T f_t(x_t,y)-\sum_{t=1}^T f_t(x_t,y_t),
\end{align*}
which is the static regret\footnote{Other works (e.g.,~\cite{meng2024proximal}) have considered objectives based on a dynamic version of the individual regrets.}  of each player taking the other's actions into account. 

Early studies focus on bilinear settings \citep{rivera2019competing, huangonline, zhang2022no} and show that static notions of performance are inadequate for capturing the complexities of online min-max problems. In particular, \cite{zhang2022no} show that the individual regrets of each player are not bounded by the static Nash equilibrium regret (\sne), which was studied by \cite{rivera2019competing}, and that no algorithm can achieve sub-linear bound for the individual regrets and \sne~simultaneously.

To get around this difficulty, \cite{zhang2022no} introduce alternative \emph{dynamic} notions of regret, namely the cumulative duality gap (\dgap) and the dynamic Nash Equilibrium regret (\dyne), which are compatible with the individual regrets.

\subsection{Our contributions}
While \cite{meng2024proximal} extend this study to more general convex-concave functions
using an online version of the proximal point method (PPM), we show (Example \ref{ex:dyne}) that even when the \dyne~is sublinear, it is possible that the average action pairs of the players never approach the \emph{cumulative} saddle point, i.e., $(x',y')$ where $$x' =\argmin_x \max_y \sum_{t=1}^T f_t(x,y),~~~~~y' = \argmax_y \min_x\sum_{t=1}^T f_t(x,y).$$ This observation suggests a fundamental distinction between minimizing the individual regrets and finding the cumulative saddle point, whereby the latter has been far less explored. In this work, we bridge this gap by providing new algorithms designed specifically to address the latter objective. We term the study dealing with this \textit{cumulative} saddle-point, online min-max optimization (OMMO).

In order to study the cumulative saddle-point we introduce two novel notions of regret, \sdg\\and~\dsp. To define these we first introduce the \emph{static gap function} $g'_t(x,y) =  f_t(x,y')- f_t(x',y)$, where the point $(x',y')$ is as defined above. In terms of $\{g'_t\}_{t=1}^T$ we have
$$\text{SDual-Gap}_T\defeq \sum_{t=1}^T g_t'(x_t,y_t)~~~\text{and}~~~\text{DSP-Reg}_T \defeq \sum_{t=1}^T g_t'(x_t,y_t)-\sum_{t=1}^T \min_{x,y} g_t'(x,y).$$

\paragraph{Static performance measures in OMMO.}
 The \sdg~can be expressed as the static regret (Eq.~\eqref{eqn:sreg}) of the functions $\{g'_t\}_{t=1}^T$ since $\min_{(x,y)\in \cx\times\cy} \sum_{t=1}^T g'_t(x,y)=0$. When all arriving functions are identical, \sdg~reduces to the standard gap function in the offline min-max setting, and it is closely related to the gap function studied in \cite{jordan2024adaptive}. We show (Theorem \ref{thm:OGDA}) that online mirror descent (OGDA) achieves an upper bound of $O(\log T)$ on \sdg~in the strongly convex-strongly concave setting, and we further provide an online min-max Newton step algorithm (\ommns) that achieves a static duality gap of $O(d\log T)$ under the \ec~assumption (Theorem \ref{thm:staticons}), with extensions to a more general online variational inequality (VI) setting (Theorem \ref{thm:final}). For the \ec~condition we show through Example \ref{ex:log} that it naturally captures an adversarial version of portfolio selection. Notably, the \sdg~is minimized by OGDA and OMMNS without the knowledge of $\{g'_t\}_{t=1}^T$ (since the point $z'$ is not known). Furthermore, we show that the time average of the players actions, $(\frac{\sum_{t=1}^T x_t}{T},\frac{\sum_{t=1}^T y_t}{T})$ converges to $z'$ in the $\ell_2$ squared-norm distance at a rate of $O(\frac{\log T}{T})$ and $O(\frac{d\log T}{T})$ for OGDA and OMMNS, respectively.
Finally, we motivate the need for alternative notions of regret for approximating cumulative saddle points by showing in Example \ref{ex:dyne} that there exists a sequence of functions $\{f_t\}_{t=1}^T$ such that the average strategies of the players may maintain a distance of $\Theta
(1)$ from the cumulative saddle-point, even when \dyne~is $O(1)$.

 \paragraph{Dynamic performance measures in OMMO.}
The~\dsp~is equivalent to the dynamic regret in the OCO setting with the functions $\{g'_t\}_{t=1}^T$. For the \dsp~we show by a reduction to OCO that OGDA and \ommns, when used as base learners in the sleeping experts framework \citep{daniely2015strongly}, achieve a dynamic saddle point regret of $O( \max\{\log T,\sqrt{TV_T\log T}\})$ and $O(d \max\{\log T,\sqrt{TV_T\log T}\})$ (Theorem \ref{thm:dynecsm-combined}) for strongly convex-strongly concave and \ec~\\functions, respectively.\footnote{Throughout, $C_T,U_T,V_T$ refer to measures of uncertainty, which we further detail in Section \ref{sec:timevarying}.} Furthermore, we show (Theorem \ref{thm:dydyup}) under the separability assumption, that \dsp~upper bounds the Dual-Gap$_T$, a key measure of performance studied under time varying zero-sum games.

\paragraph{Time varying zero-sum games.} In this setting aimed at minimizing the individual regrets, we obtain a bound of $O(U_T)$ (Theorem \ref{thm:agda2pl}) on the duality gap (\dgap) for the online AGDA algorithm under a two-sided PL condition. Extending the result of \cite{cardoso2019competing}, we further show (Theorem \ref{thm:imp}) that the incompatibility of \sne~extends to strongly convex-strongly concave and thus to \ec ~and two sided-PL class of functions.

\section{Preliminaries}
In this section we provide the key definitions of the framework, starting with global assumptions we rely on throughout the paper.

\begin{assumption}[Lipschitz continuity]\label{asmp:1}
    For convex sets $\cx$ and $\cy$, we consider functions $f:\cx\times \cy \rightarrow \mathbb{R}$ that are Lipschitz continuous with parameter $L_0$, i.e., for any two points $(x_a,y_a),(x_b,y_b) \in \cx \times \cy$, $\|f(x_a,y_a)-f(x_b,y_b)\| \leq L_0\|(x_a,y_a)-(x_b,y_b)\|$.
\end{assumption}

\begin{definition}[Saddle point]
    A saddle point of a function $f:\cx\times \cy \rightarrow \mathbb{R}$ is a point $(x^*,y^*)$ that satisfies
    $$f(x,y^*) \geq f(x^*,y^*) \geq f(x^*,y)~\forall~(x,y)\in \cx\times\cy.$$
\end{definition} 

\subsection{The online min-max setting}
The basic online framework is as follows. For $t=1,\dots,T$:

\begin{itemize}
    \item Players 1 and 2 choose actions $x_t \in \cx \subseteq \mathbb{R}^m$ and $y_t \in \cy \subseteq \mathbb{R}^n$, respectively. 

    \item The environment chooses a function $f_t(x,y):\cx\times\cy \rightarrow \mathbb{R}$. The players suffer losses of $f_t(x_t,y_t)$ and $-f_t(x_t,y_t)$, respectively, satisfying Assumption \ref{asmp:1}.

    \item The players observe the sub-gradient vectors $g_1 \in \partial_x h_1(x_t)$ and $g_2 \in \partial_y h_2(y_t)$, respectively, where $h_1(x) = f_t(x,y_t)$ and $h_2(y) = f_t(x_t,y)$.
\end{itemize}
Following the definition of \sreg, the individual regrets of each player are defined by incorporating the other player's action into the function and then comparing against the best comparator in hindsight, i.e., player 1 suffers a regret of $\text{Reg}^1_T = \sum_{t=1}^T  f_t^y(x_t)-\min_x \sum_{t=1}^T  f_t^y(x)$ where $f_t^y(x) = f_t(x,y_t)$ and player 2 suffers a regret of $\text{Reg}^2_T = \sum_{t=1}^T  f_t^x(y)-\min_y \sum_{t=1}^T  f_t^x(y)$ where $f_t^x(y) = -f(x_t,y)$. A dynamic version of these regrets defined as $\sum_{t=1}^T  f_t^y(x_t)-\sum_{t=1}^T  \min_x f_t^y(x)$ (and analogously for player 2) was considered in \cite{meng2024proximal}. The cumulative duality gap bounds the sum of the individual regrets (for both the static and dynamic notions of individual regrets).

\section{OMMO for Functions with Lower Regularity}\label{sec:dynreg}

We begin by presenting the two lower regularity conditions we consider for the OMMO setting, namely strong convexity-strong concavity and \ec.

\begin{definition}[Strong convexity-strong concavity]\label{def:strcon}
    A function $f$ is $\lambda$-strongly convex-strongly concave if for any point $(x_c,y_c)\in \cx\times\cy$ we have, $f(x_b,y_c) \geq f(x_a,y_c)+\nabla_x f(x_a,y_c)+\frac{\lambda}{2} \|x_b-x_b\|^2~\forall~x_a,x_b\in \cx$ and $f(x_c,y_a) \geq f(x_c,y_b)-\nabla_y f(x_c,y_a)+\frac{\lambda}{2} \|y_b-y_a\|^2~\forall~y_a,y_b\in \cy$
\end{definition}

Note that $\lambda=0$ corresponds to the convex-concave function class. In order to define \ec,~we first present the definition of the exponentially concave function class \citep{cesa2006prediction,hazan2016introduction}, which appears naturally in the OCO setting of portfolio selection.

\begin{definition}[Exponentially Concavity (EC)]\label{def:expc}
    A function $f(x)$ is $\alpha$-exponentially concave if $h(x) = e^{-\alpha f(x)}$ is concave.
\end{definition}
We now introduce the following generalization of exponential concavity to min-max settings.
\begin{definition}[\ec]
    A function $f$ satisfies \ec~with parameter $\alpha$ if $g_{y_c}(x) = f(x,y_c)$ is $\alpha$-exponentially concave for all $y_c \in \mathcal{Y}$ and 
    $g_{x_c}(y) = -f(x_c,y)$ is $\alpha$-exponentially concave for all $x_c \in \mathcal{X}$.
\end{definition}

Inspired by previous work on universal portfolio selection~\citep{cover1991universal,hazan2016introduction}, we motivate this definition by considering the following class of min-max EC functions, which, as we next discuss, naturally capture a variant of portfolio selection \emph{subject to adversarial manipulation of the stock price updates}.

\begin{restatable}{example}{logexample}
\label{ex:log}
The function \( f: \mathbb{R}_+^n \times \mathbb{R}_+^n \to \mathbb{R} \), 
\[
f(x,y) = -\ln \big(x^\top A (1/y)\big)
\]
for diagonal matrices $A \in \mathbb{R}_+^{n\times n}$ is \ec~with parameter $\alpha = 1$. 

\end{restatable}

\paragraph{Two-player portfolio selection.}

As a scenario inspired by portfolio selection \citep{cover1991universal,hazan2016introduction} and Example \ref{ex:log}, consider $A_t$ to be a diagonal matrix of prices of stocks, bonds or commodities, and let $\cy$ represent a player that can manipulate the prices of these stocks. Then,
\[
r_t = A_t \cdot \left(\frac{1}{y_t}\right)
\]
represents the ratio of the price of the manipulated stocks, 
\[
r_t(i) = \frac{\text{price of Stock}(i)\text{ at time } t+1}{\text{price of Stock}(i)\text{ at time } t}.
\]
The player $\cx$ first creates a portfolio based on these stocks and then $A_t$ is chosen by the environment with their wealth $W_t$ following
\[
W_{t+1} = W_t \cdot r_t^\top x_t,
\]
as in the case of portfolio selection (see, e.g., Section 4.1.2 in~\citep{hazan2016introduction}). The cumulative saddle point of $\sum_{t=1}^T x A_t(\frac{1}{y})$ therefore represents the fixed rebalanced portfolio for the $\mathcal{X}$ player and fixed manipulation for the $\cy$ player that are best in hindsight. As we will see in Section \ref{app:findingsaddle}, under our algorithm \ommns, the average actions of the players converge to the cumulative saddle-point in the \ec~setting.

\subsection{Static performance}\label{subsec:static}
In this section we consider the rates on static duality gap achieved by algorithms such as OMMNS and OGDA.
We now show that the OGDA~and \ommns~algorithms achieve a rate of $O(\log(T))$ on \sdg~in the strongly convex-strongly concave and \ec~setting.
\begin{algorithm}
\caption{Online Gradient Descent Ascent (OGDA)}
\label{alg:ogda}
\KwIn{$(x_0,y_0), (u_0^x,u_0^y) \in \mathcal{X} \times \mathcal{Y}$}
\textbf{Initialize:} $A_0=0$\\
\For{$t = 1$ \textbf{to} $T$}{
    Play $(x_t,y_t)$, observe cost $f_t(x_t,y_t)$ and $F(x_t,y_t)$\\
    Regularity update: $A_t = tI$\\
    $(u_{t+1}^x,u_{t+1}^y) = (x_t,y_t) - \frac{1}{\gamma} A_t^{-1} F(x_t,y_t)$\\
    $(x_{t+1},y_{t+1}) = \argmin_{(x,y) \in \cx\times\cy} \|(u_{t+1}^x,u_{t+1}^y) - (x,y)\|_{A_t}^2$\
}
\end{algorithm}
\begin{restatable}{theorem}{OGDATheorem}
\label{thm:OGDA}
Online gradient descent ascent (Algorithm \ref{alg:ogda}), when run on $\lambda$ strongly convex-strongly concave functions $f_t$ over domain $\cx\times \cy \subseteq \mathbb{R}^d$ with maximum operator norm $L_0$, generate action-pairs $\{x_t,y_t\}_{t=1}^T$ such that \sdg$\leq\frac{L_0^2}{\lambda}\log T $.  
\end{restatable}

While for the strongly convex-strongly concave case it is sufficient to increment $A_t$ by $I$, for the \ec~case we increment $A_t$  by the outer product $F(x,y)F(x,y)^\top$ in order to guide the \ommns~ algorithm towards the cumulative saddle point.

\begin{algorithm}
\caption{Online Min-Max Newton Step (OMMNS)}
\label{alg:lraalg2}
\KwIn{$(x_0,y_0), (u_0^x,u_0^y) \in \mathcal{X} \times \mathcal{Y}$}
\textbf{Initialize:} $A_0=I$\\
\For{$t = 1$ \textbf{to} $T$}{
    Play $(x_t,y_t)$, observe cost $f_t(x_t,y_t)$ and $F(x_t,y_t)$\\
    Regularity update: $A_t = A_{t-1} + F_t(x_t,y_t)F_t(x_t,y_t)^\top$\\
    $(u_{t+1}^x,u_{t+1}^y) = (x_t,y_t) - \frac{1}{\gamma} A_t^{-1} F(x_t,y_t)$\\
    $(x_{t+1},y_{t+1}) = \argmin_{(x,y) \in \cx\times\cy} \|(u_{t+1}^x,u_{t+1}^y) - (x,y)\|_{A_t}^2$\
}
\end{algorithm}

\begin{restatable}{theorem}{staticonsTheorem}
\label{thm:staticons}
Online min-max Newton step (Algorithm \ref{alg:lraalg2}), when run on \ec~functions $f_t$ defined over domain $\cx\times\cy \subseteq \mathbb{R}^d$ with the diameter of the domain  $\max_{(x_1,y_1),(x_2,y_2)\in \cx\times\cy }\{\|(x_1,y_1)-(x_2,y_2)\|\} \leq D$ and maximum operator norm $L_0$, generate action-pairs $\{x_t,y_t\}_{t=1}^T$ such that \sdg$\leq2d(\frac{1}{\alpha} + L_0D)\log T$.
\end{restatable}

From the definition of the static duality gap, \sdg$= \sum_{t=1}^T g_t'(x_t,y_t)$, where $g_t'(x_t,y_t) = f_t(x_t,y')-f_t(x',y_t)$. Recall that in the OCO setting, the static regret is defined as the difference in the performance between the algorithms actions and the optimal comparator in hindsight, i.e., for a series of functions $g_t$ we have that for any action pair $z=(x,y)\in \cx\times\cy$, the static regret is defined as $\text{S-Regret}_T \defeq \sum_{t=1}^T g_t(z_t) -\min_{z\in\cz} \sum_{t=1}^T g_t(z)$ in some domain $\cz$. If we consider the incoming functions to be $g_t(z) = g'_t(x,y)$ (where $z$ is the action pair $(x,y)$ and $\cz = \cx\times\cy$), then we find that the functions $g'_t(z)$ are convex in $z$ and that $\argmin_z \sum_{t=1}^T g'_t(z) = z'$, where $g'_t(z') = 0$. Thus, we have \sdg$= \sum_{t=1}^T g'_t(x_t,y_t) -\min_{(x,y)\in\cx\times\cy} \sum_{t=1}^T g'_t(x,y)$ and is analogous to \sreg. The proofs of Theorems \ref{thm:OGDA} and \ref{thm:staticons} are provided in Appendix \ref{app:static}. 

\subsubsection{Convergence to the cumulative saddle points}\label{app:findingsaddle}

In addition to obtaining small static duality gap for the strongly convex-strongly concave and \ec~settings, we may further show that the average of the players' action pairs over time, i.e. $\frac{1}{T}\sum_{t=1}^T z_t = \frac{1}{T}\sum_{t=1}^T (x_t,y_t)$, converges to the \emph{cumulative saddle-point}. 

\begin{corollary}
    For Algorithms OGDA and OMMNS, we have that:

    \begin{itemize}
        \item If the incoming functions, $\{f_t\}_{t=1}^T$ are $\lambda$-strongly convex-strongly concave, then the players actions $\{(x_t,y_t)\}_{t=1}^T$ generated by the OGDA algorithm satisfy $O\!\left(
\frac{\log T}{T}\right)\ge\lambda \left\|\frac{1}{T}\sum_{t=1}^T z_t - z'\right\|_2^2.$
\item If the incoming functions, $\{f_t\}_{t=1}^T$ are $\alpha$-\ec, then the players actions $\{(x_t,y_t)\}_{t=1}^T$ generated by the \ommns~algorithm satisfy $O\!\left(
\frac{d\log T}{T}\right)\ge\gamma \left\|\frac{1}{T}\sum_{t=1}^T z_t - z'\right\|_2^2,$ where $\gamma= \frac{1}{2}\min\{\frac{1}{L_0D},\alpha\}$ and $D$ is such that $\max_{(x_1,y_1),(x_2,y_2)\in \cx\times\cy }\{\|(x_1,y_1)-(x_2,y_2)\|\} \leq D$.
        \end{itemize}
\end{corollary}
\begin{proof}
    
\textbf{Strongly convex-strongly concave setting.} Since $g_t'(z)$ are strongly convex in $z$, when the functions $f_t(x,y)$ are strongly convex-strongly concave we have that,
\[
\frac{1}{T}\sum_{t=1}^T g_t'(z_t)
- \frac{1}{T}\sum_{t=1}^T g_t'(z')
- \frac{1}{T}\sum_{t=1}^T 
\left\langle g_t'(z'),\, z_t - z' \right\rangle
\ge
\sum_{t=1}^T \lambda \frac{\|z_t - z'\|_2^2}{T}
\ge
\lambda \left\|\frac{1}{T}\sum_{t=1}^T z_t - z'\right\|_2^2.
\]

where the last inequality follows from Jensen's inequality on the convex function $\|z-z'\|_2^2$. From the bound on \sdg~we have,
\[
O\!\left(
\frac{\log T}{T}\right)
-
\frac{1}{T}\sum_{t=1}^T
\left\langle \nabla_z g_t'(z'),\, z_t - z' \right\rangle
\ge
\lambda \left\|\frac{1}{T}\sum_{t=1}^T z_t - z'\right\|_2^2.
\]

Finally since $z' = \argmin_{\mathcal{Z}} \sum_{t=1}^T g_t'(z)$ we have that $\sum_{t=1}^T  \left\langle \nabla_z g_t'(z'),\, z - z' \right\rangle \geq 0,\forall~z\in\mathcal{Z}$ which implies,
\[
O\!\left(
\frac{\log T}{T}\right)\ge\lambda \left\|\frac{1}{T}\sum_{t=1}^T z_t - z'\right\|_2^2,
\]

and demonstrates that $\frac{1}{T}\sum_{t=1}^T z_t$ approaches $z'$ in the $\ell_2$-norm.

\textbf{Min-max EC~setting.}
 Since $g_t'(z)$ are exponentially concave when the functions $f_t(x,y)$ are \ec, we have that,
\[
\frac{1}{T}\sum_{t=1}^T g_t'(z_t)
- \frac{1}{T}\sum_{t=1}^T g_t'(z')
- \frac{1}{T}\sum_{t=1}^T 
\left\langle \nabla_z g_t'(z'),\, z_t - z' \right\rangle
\ge
\sum_{t=1}^T \gamma \frac{\|z_t - z'\|_{F(z_t)F(z_t)^\top}^2}{T}
\]
where the last inequality follows from Jensen's inequality on the convex function $\|z-z'\|_2^2$. From the bound on \sdg, we have
\[
O\!\left(
\frac{\log T}{T}\right)
-
\frac{1}{T}\sum_{t=1}^T
\left\langle \nabla_z g_t'(z'),\, z_t - z' \right\rangle
\ge
\sum_{t=1}^T \gamma \frac{\|z_t - z'\|_{F(z_t)F(z_t)^\top}^2}{T}
\]
Since $z' = \argmin_{\mathcal{Z}} \sum_{t=1}^T g_t'(z)$ we have that $\sum_{t=1}^T \left\langle  \nabla_z g_t'(z'),\, z - z' \right\rangle \geq 0,\forall~z\in\mathcal{Z}$ which implies,
\[
O\!\left(
\frac{\log T}{T}\right)\geq \frac{\gamma}{T}\sum_{t=1}^T
\|z - z_t\|_{F(z_t)F(z_t)^\top}^2 \geq \left(
\frac{\gamma}{T}\sum_{t=1}^T \|z' - z_t\|_2^2
\right)
\cdot
\min_{1 \le t \le T}
\sigma_{\min}\!\left(F(z_t)F(z_t)^\top\right).
\]
Thus, as long as $\min_{1 \le t \le T}
\sigma_{\min}\!\left(F(z_t)F(z_t)^\top\right)> 0,$
we have that
$
O\left(
\frac{\log T}{T}\right)\geq \left(
\frac{\gamma}{T}\sum_{t=1}^T \|z' - z_t\|_2^2
\right)
\cdot
\min_{1 \le t \le T}
\sigma_{\min}\!\left(F(z_t)F(z_t)^\top\right)$
which implies,
$$O\!\left(
\frac{\log T}{T}\right)\ge \left\|\frac{1}{T}\sum_{t=1}^T z_t - z'\right\|_2^2,$$
and thus we have that $\frac{1}{T}\sum_{t=1}^T z_t$ converges to $z'$ in the $\ell_2$-norm.
\end{proof}

As a consequence of this convergence in the case of our adversarial portfolio selection example (which satisfies \ec), we thus have that the actions of $\cx$ and $\cy$ converge to the best fixed rebalanced portfolio under the best fixed manipulation factors in hindsight.

\subsubsection{\sdg~vs. \dyne}

The \dyne~was proposed by \cite{zhang2022no} as a regret measure to guarantee small individual regrets after it was shown that small \sne does not guarantee the same by \cite{cardoso2019competing}. We now show through an example that we can have constant (and thus sublinear in $T$) \dyne~even when the average strategies of the players do not approach the cumulative saddle-point. In contrast, we have seen in, highlighting the importance of \sdg, since we are guaranteed that the average of the players actions converges to the cumulative saddle-point if \sdg~is sublinear. We begin by defining \dyne.

Interestingly, a drawback to \dyne~is that---in contrast to the \sdg~case---it is possible for a sequence of players' actions to obtain small \dyne, while the average does \emph{not} converge to the cumulative saddle point.

\begin{definition}[Dynamic Nash Equilibrium Regret, $\text{DNE-Reg}_T$]
The \emph{dynamic Nash equilibrium regret} is defined as
\begin{equation}
    \text{DNE-Reg}_T 
    \;\coloneqq\;
    \Biggl|
    \sum_{t=1}^{T} f_t(x_t, y_t)
    \;-\;
    \sum_{t=1}^{T} 
    \min_{x} \max_{y} f_t(x, y)
    \Biggr|.
\end{equation}
\end{definition}

The following example demonstrates the distinction between the individual regret objectives and that of the cumulative saddle-point whereby \dyne~, a regret notion designed in order to obtain low individual regrets can be small (e.g., $O(1)$) even when the time-average of the players strategies does not approach the cumulative saddle-point.
\begin{example}\label{ex:dyne}
    Consider the sequence of functions $f_t(x,y) = (x-x^*_t)^2-(y-y^*_t)^2$ where $x_t^* = x_t+\frac{(-1)^t+1}{2}(2\mathbf{1}_{\{x_t < 0\}}-1)$ and $y_t^* =  y_t-\frac{(-1)^t-1}{2}(2\mathbf{1}_{\{y_t < 0\}}-1)$.
    The cumulative saddle point of the functions is $(\frac{\sum x_t^*}{T},\frac{\sum y_t^*}{T})$ and its distance from the average of the strategies is $\left(\frac{\sum_t (x_t^* - x_t)}{T}, \frac{\sum_t (y_t^* - y_t)}{T}\right)
= \left(\frac{\sum_{t \in \text{even}[1,T]} \big(2\mathbf{1}_{\{x_t < 0\}} - 1\big)}{T},
\frac{\sum_{t \in \text{odd}[1,T]} \big(2\mathbf{1}_{\{y_t < 0\}} - 1\big)}{T}\right)$
 which may be $(\Theta(1),\Theta(1))$ even when \dyne$\leq1$. 
\end{example}
The calculations for the above example are detailed in Appendix \ref{app:dyne}.
\subsection{Dynamic regret}
From our previous discussion, we know that the algorithms OGDA and \ommns ~achieve low \sdg~and, in turn, achieve low \sreg~for the correct mapping (without the knowledge of the cumulative saddle point $z'=(x',y')$). Building on this, we now provide an algorithm to obtain low \emph{dynamic regret} (\dsp),~which we recall is the dynamic version of the \sdg~discussed in the previous section. 

The notion of dynamic regret in OCO literature, for a sequence of functions $\{g_t\}_{t=1}^T$, is defined as \dreg$\coloneqq \sum_{t=1}^T g_t(z_t)-\sum_{t=1}^T \min_{z\in\cz} g_t(z)$ for some domain $\cz$. Our notion of dynamic regret for the online min-max optimization setting, the dynamic saddle point regret \dsp, is inspired by \dreg~and is analogous to \dsp~with the functions $g_t(z) = g'_t(z) = f_t(x',y)-f(x,y')$ (as before $z=(x,y)$ and $\cz = \cx\times\cy$). Based on these observations, we arrive at the following key results.
\begin{theorem}[Informal]
\label{thm:dynecsm-combined}
\leavevmode
\begin{itemize}
    \item If $f_t$ are strongly convex-strongly concave, then there is an algorithm that, upon using OGDA as its base learner, provides for
    $$\text{DSP-Reg}_T \leq O(\max\{\log T,\sqrt{TV_T\log T}\}).$$
    \item If $f_t$ are \ec, then there is an algorithm that, using \ommns~as its base learner, provides for
    $$
    \text{DSP-Reg}_T \leq O(d \max\{\log T,\sqrt{TV_T\log T}\}).
    $$
\end{itemize}

Here, $V_T = \sum_{t=1}^T \max_{(x,y)\in \cx\times\cy} |g_t'(x,y)-g_{t-1}'(x,y)|$.
\end{theorem}

In order to establish these results, we adapt machinery developed for obtaining low \dreg~in OCO to obtain low \dsp. In particular, we  use a sleeping-experts framework \citep{freund1997using,hazan2009efficient,daniely2015strongly,jun2017improved,zhang2017strongly}. Following the IFLH framework of \cite{zhang2018dynamic}, we use either \ommns~or OGDA as the base learner (algorithms that obtain small \sdg) in a min-max version of the framework and bound \dsp~via \sasp. We provide this MMFLH framework in the appendix for completion.\footnote{\cite{zhang2022no} use two different meta-learning frameworks for each player for bounds on \dgap.} The \sasp~is defined analogous to  \sareg$(\tau)\defeq \max_{[s,s+\tau-1]\subseteq [T]}\big(\sum_{t=s}^{s+\tau-1}g_t(z_t)$\\  $-\min_{z\in \cx\times\cy}\sum_{t=s}^{s+\tau-1} g_t(z)\big)$  with $g_t$ substituted with $g'_t$, and the action pairs $(x_t,y_t),(x,y)$ denoted by $z_t,z\in \cx\times\cy$ . The proof of Theorem \ref{thm:dynecsm-combined} is provided in Appendix~\ref{app:dynamic}. It utilize the fact that the functions $g_t$ are exponentially concave and strongly convex when the functions $f_t$ are  exponentially concave and strongly convex-strongly concave, respectively, and that the base learners, \ommns~and OGDA achieve logarithmic static-regret, i.e., \sdg.

Finally, we show that under the separability assumption, the \dgap~is bounded from above by \dsp. We first define the cumulative duality gap, \dgap~in terms of \textit{dynamic gap functions} $g^*_t(x,y) = f_t(x,y_t^*(x))-f_t(x_t^*(y),y)$, where $x_t^*(y) \coloneqq \argmin_{x\in \mathcal{X}} f_t(x,y)$~\text{and} $y_t^*(x) \coloneqq \argmax_{y\in \mathcal{Y}} f_t(x,y)$.

\begin{definition}[Cumulative Duality Gap, $\text{Dual-Gap}_T$]\label{def:dualgap}
The \emph{cumulative duality gap} of the sequence $\{(x_t, y_t)\}_{t=1}^T$ is defined as the sum of the dynamic gap functions:
\begin{equation}
    \text{Dual-Gap}_T 
    \;\coloneqq\;
    \sum_{t=1}^{T} g_t^*(x_t, y_t)
    \;=\;
    \sum_{t=1}^{T}
    \Big[
        f_t\!\big(x_t,\, y_t^*(x_t)\big)
        \;-\;
        f_t\!\big(x_t^*(y_t),\, y_t\big)
    \Big].
\end{equation}
\end{definition}

The \dyne~and the individual regrets (both static and dynamic versions) are upper bounded by \dgap. The fact that \dyne~$\leq$ \dgap~was proven for the bilinear case in \cite{zhang2022no}, and we show in Appendix \ref{app:sec2} that this holds for the more general convex-concave case. Note also that the sum of individual regrets in the static sense is less than the sum in the dynamic sense (which is equal to \dgap). Thus~\dgap~provides a direct upper bound on the individual regrets and serves as a representative notion of regret in the time-varying zero-sum games setting.

We now present our theorem, the proof of which is presented in Appendix \ref{app:rels}.

\begin{theorem}[Informal]
\label{thm:dydyup}
For functions $f_t$ that are convex-concave and can be represented as $f(x,y) = h_1(x)-h_2(y)$, where $h_1(x)$ and $h_2(y)$ are convex in $x$ and $y$, respectively, we have that~\dgap~$\leq$~\dsp.
\end{theorem}

\subsection{Time-varying variational inequality objective under lower regularity}

Our techniques also lend themselves naturally to a setting of 
time-varying variational inequalities (VIs), as discussed by~\cite{anagnostides2024convergence}. In particular, we consider the setting where the operator $F$ satisfies a certain lower regularity similar to strong-monotonicity. In this setting, we receive an operator $F_t$ at each time step $t$ and the aim is to bound $\sum_{t=1}^T \ang{F_t(z),z_t-z}$, for all $z \in \cz$, which we view as an analogue of cumulative regret for the sequence of operators $\{F_t\}_{t=1}^T$. We propose an algorithm (Algorithm \ref{alg:mainalgvi}) that achieves a rate of $O(d\log T)$ on this online VI objective under the novel lower regularity condition on the operator $F$, inspired by the \ec~condition defined as follows.

\begin{definition}\label{def:lowrankmatrix}
    For a split $s = (s_1,s_2,..,s_m)$ with $m\leq d$ the variable matrix $M(z)$ is defined as,
    $M_s(z) = \begin{bmatrix}
        F_{s_1}(z_{s_1})F_{s_1}(z_{s_1})^\top & 0  & \dots &0\\
        0 & F_{s_2}(z_{s_2})F_{s_2}(z_{s_2})^\top  & \dots &0\\
                        &  &\ddots\\
        0 & 0  &\dots & F_{s_m}(z_{s_m})F_{s_m}(z_{s_m})^\top
    \end{bmatrix}$
    where $z = (z_{s_1},z_{s_2},\dots,z_{s_m})$ and $F(z) = (F_{s_1}(z_{s_1}),F_{s_2}(z_{s_2}),\dots, F_{s_m}(z_{s_m}))$ is an operator mapping $\mathbb{R}^d$ to $\mathbb{R}^d$  where $\mathbb{R}^d = \mathbb{R}^{s_1} \times\mathbb{R}^{s_2} \dots
    \times\mathbb{R}^{s_m}$.
\end{definition}
\begin{definition}[Low-rank monotonicity]\label{def:vilr}
We say $F$ is \emph{low-rank monotone} with respect to split $s \in \mathbb{N}^m$ if, for all $z_a, z_b \in \cz$,
    $$ \ang{F(z_b)-F(z_a),z_b-z_a} \geq \gamma \|z_b-z_a\|_{M_s(z_b)}^2,$$
    where  $M_s(z)$ is the block diagonal matrix as defined in Definition~\ref{def:lowrankmatrix}.
\end{definition}
The above condition reduces to the structure of multi-agent online games in \cite{jordan2024adaptive} if $F_{s_i}(z)$ are gradients of functions $f_i : \mathbb{R}^{s_i}\rightarrow \mathbb{R}^{s_i} $ i.e., $F_{s_i}(u) = \nabla_u f_{s_i}(u)$. In addition, the number of splits determine the rank of the matrix $M_s(z)$ and in turn determines the computation cost of performing the update. 

\begin{theorem}[Informal]
\label{thm:final}
The low-rank Newton's method algorithm (Algorithm \ref{alg:mainalgvi}) obtains a regret of $O(\log T)$ on the online VI objective:
\begin{equation}\label{onlineVI}
    \sum_{t=1}^T\ang{F_t(z),z_t-z} \leq O(d\log T)~~~~~~~~~\forall~z\in \cz\subseteq \mathbb{R}^d,
\end{equation}
for operators satisfying the condition in Definition \ref{def:vilr} over a domain with bounded diameter. 
\end{theorem}

Additional details and the proof of Theorem \ref{thm:final} is provided in Appendix \ref{app:vi}.
As a corollary we have that in the offline setting, i.e. $F_t = F$ for all $t$, the online VI optimization achieves a bound of $O(\frac{d\log(T)}{T})$ on the VI objective for operators satisfying Definition \ref{def:vilr}, as opposed to a bound of $O(\frac{1}{\sqrt{T}})$ in the standard bounded operator-norm setting \citep{nemirovski2004prox}. This potentially faster rate is due to the lower-regularity assumption and how the algorithm may leverage this structure.

\section{Time-Varying Zero-Sum Games}\label{sec:timevarying}

For this section, we drop the convex-concave assumption on the functions and assume that a certain two-sided PL inequality holds along with smoothness of the functions. Under the two-sided PL inequality and smoothness, the AGDA algorithm achieves linear convergence to a saddle point (\cite{yang2020global}). The two-sided PL condition is akin to a centered notion of strong monotonicity, and thus linear rates analogous to strong monotonicity are possible. Examples of the two-sided PL condition are provided in Appendix \ref{app:2sidedPL}.

\begin{definition} [Two-sided PL condition]\label{def:2sidedePL}
A continuously differentiable function $f(x,y)$ satisfies the two-sided PL condition if there exist constants $\mu_1, \mu_2 > 0$ such that, $\|\nabla_x f(x,y)\|^2 \geq 2\mu_1[f(x,y) - \min_x f(x,y)], \forall x,y,$ and $\|\nabla_y f(x,y)\|^2 \geq 2\mu_2[\max_y f(x,y) - f(x,y)], \forall x,y$.
\end{definition}
Exponential convergence of the alternating gradient descent ascent algorithm under this setting was shown in \cite{yang2020global}. Examples satisfying the two-sided PL condition are provided in the Appendix \ref{app:2sidedPL}.
We build upon these results and use them in the online min-max setting in the next section. Furthermore for this section, we will assume that the functions are differentiable and have bounded smoothness. 
\begin{definition}
  A function $f$ is $L_1$-smooth if exists a positive constant $L_1 > 0$ such that
\[
\max\left\{\|\nabla_y f(x_a, y_a) - \nabla_y f(x_b, y_b)\|, \|\nabla_x f(x_a, y_a) - \nabla_x f(x_b, y_b)\|\right\} 
\leq  L_1 \left(\|x_a - x_b\| + \|y_a - y_b\|\right),
\]
holds for all $x_a, x_b \in \cx$ and $y_a, y_b \in \cy$.
\end{definition}

While for the convex-concave setting considered in the previous section, it is guaranteed that a saddle-point exists within $\cx\times\cy$, it is not so in the two-sided PL setting. In order to properly specify our two-sided PL condition, we also rely on the following assumption of the existence of a saddle-point, also presented by \cite{yang2020global}.
\begin{assumption}[Existence of saddle point]\label{asmp:2}
    There exists a saddle-point for all functions $f_t$. We also assume that for all $f_t$, for any fixed $y, \min_{x\in\cx} f(x,y)$ has a nonempty solution set and a optimal value, and for any fixed x, $max_{y\in \cy }f(x, y)$ has a nonempty solution set and a finite optimal value.
\end{assumption}

Inspired by the linear convergence of Alternating Gradient Descent Ascent (AGDA) in the two-sided PL setting (and thus its suitability in online learning), we design the following online AGDA algorithm. 

\begin{algorithm}[h]
\caption{Online Alternating Gradient Descent Ascent}
\label{alg:mainalg1}
\KwIn{$(x_0, y_0) \in \mathcal{X} \times \mathcal{Y}$, $\tau_1$, $\tau_2$}
\For{$t = 0$ \textbf{to} $T$}{
    Play action $(x_t,y_t)$ and set $(x_{t,0},y_{t,0}) = (x_t,y_t)$\\
    Observe $f_t$ and suffer losses $f_t(x_t,y_t),-f_t(x_t,y_t)$.\\
    Set $K_t = \max\{1,\lceil\log_\rho \frac{\max_y f_t(x_t,y)-\min_x f(x,y_t)}{4\beta (\max_y f(x_t,y)-\min_x\max_y f(x,y)+\frac{\max_y f(x_t,y)-f(x_t,y_t)}{10})}\rceil\}$\\
    \For{$k = 0$ \textbf{to} $K_t-1$}{
        $x_{t,k+1} = x_{t,k} - \tau_1 \nabla_x f(x_{t,k}, y_{t,k})$\\
        $y_{t,k+1} = y_{t,k} + \tau_2 \nabla_y f(x_{t,k+1}, y_{t,k})$
    }
    $z_{t+1} = (x_{t,K}, y_{t,K})$\;
}
\end{algorithm}

We now provide a bound on the cumulative duality gap (Definition~\ref{def:dualgap}) using the online AGDA algorithm for functions satisfying the two-sided PL condition. As \dgap~is a dynamic measure, the bounds on it are presented in terms of variation in the objective functions $\{f_t\}_{t=1}^T$, referred to as Dynamic Regret with Variation in Utilities (DRVU). In particular, we rely on the following variation measures.

\begin{align*}
C_T 
&= \sum_{t=1}^T 
\Big(
    \|x^*_{t+1}(y_{t+1}) - x^*_t(y_t)\|
    + 
    \|y^*_{t+1}(x_{t+1}) - y^*_t(x_t)\|
\Big), 
\\[4pt]
\Delta_T 
&= \sum_{t=1}^T 
\|(x_{t+1}, y_{t+1}) - (x_t, y_t)\|,
\qquad
V_T'
= \sum_{t=1}^T 
\max_{(x, y) \in \mathcal{X} \times \mathcal{Y}}
\big| f_t(x, y) - f_{t-1}(x, y) \big|,
\\[4pt]
C_T' 
&= \sum_{t=1}^T 
\Big(
    \|x^*_{t+1}(y_{t+1}) - x^*_t(y_{t+1})\|
    + 
    \|y^*_{t+1}(x_{t+1}) - y^*_t(x_{t+1})\|
\Big).
\end{align*}

In the above expressions, we let $x_t^*(y) \defeq \argmin_x f_t(x,y)$ and $y_t^*(x) \defeq \argmax_y f_t(x,y)$. We therefore arrive at the following theorem.
\begin{restatable}{theorem}{agdaplTheorem}
\label{thm:agda2pl}
The sequence of actions $\{(x_{t,K}, y_{t,K})\}_{t=1}^T$ generated by Online AGDA (Algorithm \ref{alg:mainalg1}) satisfies \dgap$\leq U_T+2(g^*_1(x_1,y_1)-g^*_T(x_{T+1},y_{T+1}))$ where $U_T = \max_{(x,y)\in \cx\times\cy} \sum_{t=1}^T$\\$  g^*_{t+1}(x,y) -g^*_t(x,y)$, and $g^*_t(\cdot,\cdot)$ are the dynamic gap functions.  Furthermore, we have that
$\text{Dual-Gap}_T \leq O(\min \{U_T,\{V_T'+ \Delta_T+C_T\},\{V_T'+C_T'\}\}).$
\end{restatable}

Note that since all strongly convex-strongly concave functions satisfy the two-sided PL-inequality, the above bound also holds for this setting. For strongly convex-strongly concave functions, \cite{anagnostides2024convergence} bound $\max\{\text{Reg}^1_T,\text{Reg}^2_T\}$ by $O(\sum_{t=1}^T (\|(\hat x_{t+1},\hat y_{t+1})-(\hat x_{t},\hat y_{t})\|_2^2+\max_{(x,y)\in \cx\times\cy}\|F_t(x,y)-F_{t-1}(x,y)\|))$, where $(\hat x_{t},\hat y_{t})$ is the saddle point of $f_t(x,y)$. The proofs of the statements in this section are provided in Appendix \ref{app:sec3n4}.

Finally, we may also consider an alternative \emph{static} notion of Nash equilibrium regret, which has been studied in the context of time-varying zero-sum games~\citep{cardoso2019competing}. We note that \sne~is related to \sdg, and in terms of the cumulative saddle point $(x',y')$, it can be written as $|\sum_{t=1}^T f_t(x_t,y_t)-\sum_{t=1}^T f_t(x',y')|$ for convex-concave functions $\{f_t\}_{t=1}^T$.

\begin{definition}[Static Nash Equilibrium Regret]
The \emph{static Nash equilibrium regret} ($\text{SNE-Reg}_T$) is defined as
\begin{equation}
    \text{SNE-Reg}_T 
    \;:=\;
    \Biggl|
    \sum_{t=1}^{T} f_t(x_t, y_t)
    \;-\;
    \min_{x} \max_{y} \sum_{t=1}^{T} f_t(x, y)
    \Biggr|.
\end{equation}
\end{definition}

A drawback to this notion is that it is \emph{incompatible} with the individual regrets for the bilinear case~\citep{cardoso2019competing}, in the sense that no algorithm can achieve sublinear bounds for the individual regrets and the static Nash equilibrium regret simultaneously.
We thus extend the incompatibility of \sne~with individual regrets to  the case of strongly convex-strongly concave and \ec~functions, as presented in the following theorem.

\begin{theorem}[Informal]
\label{thm:imp}
 When the domain $\cx\times\cy$ is a two-dimensional simplex, there exists a sequence of functions, all either strongly convex-strongly concave or \ec, for which there is no sequence $\{x_t,y_t\}_{t=1}^T$ that satisfies $\textit{SNE-}Reg_T \leq o(T)$, $Reg_T^1 \leq o(T)$, and $Reg_T^2 \leq o(T)$ simultaneously.
\end{theorem}

\section{Conclusion}

In this work, we present an alternative approach for online zero-sum games whereby we consider the cumulative saddle-point as the main quantity of interest. 
Specifically, we introduce novel regret notions, \sdg~and \dsp, which are designed around competing against the cumulative saddle point of the sequence of functions, and we further present the algorithms OGDA and \ommns, which obtain sub-linear upper bounds on these regrets. As a consequence of our approach, we show that \sdg~captures a natural two-player version of portfolio selection. We believe this interpretation offers interesting possibilities in terms of modeling phenomenon such as market manipulation, though we leave a more thorough exploration of this direction to future work. In addition, for bounding \dsp, we utilize the min-max follow the leader history (MMFLH), a min-max version of the improved FLH (\cite{zhang2018dynamic}), with OGDA and \ommns~as the base learners, though we leave open the problem of obtaining upper-bounds for \dsp~without knowledge of the cumulative saddle-point. Finally, we extend the incompatibility of \sne~and individual regrets to the strongly convex-strongly concave setting, thereby highlighting their unsuitability under the convex-concave, \ec~and two-sided PL settings.


\bibliography{online}

@article{nemirovski2004prox,
  title={Prox-method with rate of convergence O (1/t) for variational inequalities with Lipschitz continuous monotone operators and smooth convex-concave saddle point problems},
  author={Nemirovski, Arkadi},
  journal={SIAM Journal on Optimization},
  volume={15},
  number={1},
  pages={229--251},
  year={2004},
  publisher={SIAM}
}

@inproceedings{Zhang2019PolicyOP,
  title={Policy Optimization Provably Converges to Nash Equilibria in Zero-Sum Linear Quadratic Games},
  author={K. Zhang and Zhuoran Yang and Tamer Başar},
  booktitle={Neural Information Processing Systems},
  year={2019},
  url={https://api.semanticscholar.org/CorpusID:173991042}
}

@article{Cai2019OnTG,
  title={On the Global Convergence of Imitation Learning: A Case for Linear Quadratic Regulator},
  author={Qi Cai and Mingyi Hong and Yongxin Chen and Zhaoran Wang},
  journal={ArXiv},
  year={2019},
  volume={abs/1901.03674},
  url={https://api.semanticscholar.org/CorpusID:57825744}
}

@inproceedings{Fazel2018GlobalCO,
  title={Global Convergence of Policy Gradient Methods for the Linear Quadratic Regulator},
  author={Maryam Fazel and Rong Ge and Sham M. Kakade and Mehran Mesbahi},
  booktitle={International Conference on Machine Learning},
  year={2018},
  url={https://api.semanticscholar.org/CorpusID:51881649}
}

@inproceedings{huangonline,
  title={Online Min-max Optimization: Nonconvexity, Nonstationarity, and Dynamic Regret},
  author={Huang, Yu and Cheng, Yuan and Liang, Yingbin and Huang, Longbo},
  booktitle={OPT 2022: Optimization for Machine Learning (NeurIPS 2022 Workshop)}
}

@inproceedings{zhang2022no,
  title={No-regret learning in time-varying zero-sum games},
  author={Zhang, Mengxiao and Zhao, Peng and Luo, Haipeng and Zhou, Zhi-Hua},
  booktitle={International Conference on Machine Learning},
  pages={26772--26808},
  year={2022},
  organization={PMLR}
}

@inproceedings{zhao2021improved,
  title={Improved analysis for dynamic regret of strongly convex and smooth functions},
  author={Zhao, Peng and Zhang, Lijun},
  booktitle={Learning for Dynamics and Control},
  pages={48--59},
  year={2021},
  organization={PMLR}
}

@article{meng2024proximal,
  title={Proximal Point Method for Online Saddle Point Problem},
  author={Meng, Qing-xin and Liu, Jian-wei},
  journal={arXiv preprint arXiv:2407.04591},
  year={2024}
}

@article{yang2020global,
  title={Global convergence and variance reduction for a class of nonconvex-nonconcave minimax problems},
  author={Yang, Junchi and Kiyavash, Negar and He, Niao},
  journal={Advances in Neural Information Processing Systems},
  volume={33},
  pages={1153--1165},
  year={2020}
}

@article{rivera2019competing,
  title={Competing Against Equilibria in Zero-Sum Games with Evolving Payoffs},
  author={Rivera Cardoso, Adrian and Abernethy, Jacob and Wang, He and Xu, Huan},
  journal={arXiv e-prints},
  pages={arXiv--1907},
  year={2019}
}

@article{hazan2016introduction,
  title={Introduction to online convex optimization},
  author={Hazan, Elad},
  journal={Foundations and Trends{\textregistered} in Optimization},
  volume={2},
  number={3-4},
  pages={157--325},
  year={2016},
  publisher={Now Publishers, Inc.}
}

@article{jordan2024adaptive,
  title={Adaptive, doubly optimal no-regret learning in strongly monotone and exp-concave games with gradient feedback},
  author={Jordan, Michael and Lin, Tianyi and Zhou, Zhengyuan},
  journal={Operations Research},
  year={2024},
  publisher={INFORMS}
}

@article{zhang2018adaptive,
  title={Adaptive online learning in dynamic environments},
  author={Zhang, Lijun and Lu, Shiyin and Zhou, Zhi-Hua},
  journal={Advances in Neural Information Processing Systems},
  volume={31},
  year={2018}
}

@inproceedings{abernethy2008optimal,
  title={Optimal strategies and minimax lower bounds for online convex games},
  author={Abernethy, Jacob and Bartlett, Peter L and Rakhlin, Alexander and Tewari, Ambuj},
  booktitle={Proceedings of the 21st Annual Conference on Learning Theory},
  pages={414--424},
  year={2008}
}

@inproceedings{hsieh2021adaptive,
  title={Adaptive learning in continuous games: Optimal regret bounds and convergence to nash equilibrium},
  author={Hsieh, Yu-Guan and Antonakopoulos, Kimon and Mertikopoulos, Panayotis},
  booktitle={Conference on Learning Theory},
  pages={2388--2422},
  year={2021},
  organization={PMLR}
}

@article{cesa2012mirror,
  title={Mirror descent meets fixed share (and feels no regret)},
  author={Cesa-Bianchi, Nicolo and Gaillard, Pierre and Lugosi, G{\'a}bor and Stoltz, Gilles},
  journal={Advances in Neural Information Processing Systems},
  volume={25},
  year={2012}
}

@inproceedings{zhang2018dynamic,
  title={Dynamic regret of strongly adaptive methods},
  author={Zhang, Lijun and Yang, Tianbao and Zhou, Zhi-Hua and others},
  booktitle={International Conference on Machine Learning},
  pages={5882--5891},
  year={2018},
  organization={PMLR}
}

@inproceedings{cardoso2019competing,
  title={Competing against nash equilibria in adversarially changing zero-sum games},
  author={Cardoso, Adrian Rivera and Abernethy, Jacob and Wang, He and Xu, Huan},
  booktitle={International Conference on Machine Learning},
  pages={921--930},
  year={2019},
  organization={PMLR}
}

@inproceedings{zinkevich2003online,
  title={Online convex programming and generalized infinitesimal gradient ascent},
  author={Zinkevich, Martin},
  booktitle={Proceedings of the 20th International Conference on Machine Learning},
  pages={928--936},
  year={2003}
}

@inproceedings{mokhtari2016online,
  title={Online optimization in dynamic environments: Improved regret rates for strongly convex problems},
  author={Mokhtari, Aryan and Shahrampour, Shahin and Jadbabaie, Ali and Ribeiro, Alejandro},
  booktitle={2016 IEEE 55th Conference on Decision and Control (CDC)},
  pages={7195--7201},
  year={2016},
  organization={IEEE}
}

@article{besbes2015non,
  title={Non-stationary stochastic optimization},
  author={Besbes, Omar and Gur, Yonatan and Zeevi, Assaf},
  journal={Operations Research},
  volume={63},
  number={5},
  pages={1227--1244},
  year={2015},
  publisher={INFORMS}
}

@article{zhang2017strongly,
  title={Strongly adaptive regret implies optimally dynamic regret},
  author={Zhang, Lijun and Yang, Tianbao and Jin, Rong and Zhou, Zhi-Hua},
  journal={arXiv preprint arXiv:1701.07570},
  year={2017}
}

@inproceedings{daniely2015strongly,
  title={Strongly adaptive online learning},
  author={Daniely, Amit and Gonen, Alon and Shalev-Shwartz, Shai},
  booktitle={International Conference on Machine Learning},
  pages={1405--1411},
  year={2015},
  organization={PMLR}
}

@article{cover1991universal,
  title={Universal portfolios},
  author={Cover, Thomas M},
  journal={Mathematical Finance},
  volume={1},
  number={1},
  pages={1--29},
  year={1991},
  publisher={Wiley Online Library}
}

@article{anagnostides2024convergence,
  title={On the convergence of no-regret learning dynamics in time-varying games},
  author={Anagnostides, Ioannis and Panageas, Ioannis and Farina, Gabriele and Sandholm, Tuomas},
  journal={Advances in Neural Information Processing Systems},
  volume={36},
  year={2024}
}

@book{cesa2006prediction,
  title={Prediction, learning, and games},
  author={Cesa-Bianchi, Nicolo and Lugosi, G{\'a}bor},
  year={2006},
  publisher={Cambridge University Press}
}

@inproceedings{hazan2009efficient,
  title={Efficient learning algorithms for changing environments},
  author={Hazan, Elad and Seshadhri, Comandur},
  booktitle={Proceedings of the 26th Annual International Conference on Machine Learning},
  pages={393--400},
  year={2009}
}

@inproceedings{jun2017improved,
  title={Improved strongly adaptive online learning using coin betting},
  author={Jun, Kwang-Sung and Orabona, Francesco and Wright, Stephen and Willett, Rebecca},
  booktitle={Artificial Intelligence and Statistics},
  pages={943--951},
  year={2017},
  organization={PMLR}
}

@inproceedings{karimi2016linear,
  title={Linear convergence of gradient and proximal-gradient methods under the {P}olyak-{\L}{}ojasiewicz condition},
  author={Karimi, Hamed and Nutini, Julie and Schmidt, Mark},
  booktitle={Machine Learning and Knowledge Discovery in Databases: European Conference, ECML PKDD 2016, Riva del Garda, Italy, September 19-23, 2016, Proceedings, Part I 16},
  pages={795--811},
  year={2016},
  organization={Springer}
}

@inproceedings{freund1997using,
  title={Using and combining predictors that specialize},
  author={Freund, Yoav and Schapire, Robert E and Singer, Yoram and Warmuth, Manfred K},
  booktitle={Proceedings of the twenty-ninth annual ACM Symposium on Theory of Computing},
  pages={334--343},
  year={1997}
}
\bibliographystyle{plainnat}

\newpage

\appendix
\addcontentsline{toc}{section}{Appendix}
\part{Appendix} 
\parttoc 
\section{Additional Lemmas and Theorems}\label{app:supportive}

In this section of the appendix we present additional useful results used later to bound the \sdg~and the \dsp. We start with defining the min-max (sub) operator which will be useful throughout as an intermediary object in the proofs.

\begin{definition}[min-max sub-operators]
    Consider the function $f(x,y)$, a min-max operator is defined as the concatenation of sub-gradients,
    $$F(x_0,y_0) = (g_1,-g_2)$$
     where $g_1\in \delta_{x_0}$, $g_2\in \delta_{y_0}$ and $\delta_{x_0},\delta_{y_0}$ are sub differentials of the function $h_1(x) = f(x,y_0)$ at $x=x_0$ and of  $h_2(y) = f(x_0,y)$ at $y=y_0$ respectively. When the function $f$ is differentiable we have that $F(x_0,y_0) = (\nabla_x f(x_0,y_0),-\nabla_y f(x_0,y_0))$. We will refer to these as simply \textbf{operators} $F$, throughout the paper.
\end{definition}

We now provide bounds on $f(x_a,y_b)-f(x_b,y_a)$ for any general points $(x_a,y_a)$ and $(x_b,y_b)$ in terms of the min-max operator $F(z) = (\nabla_x f(x,y),-\nabla_y f(x,y))$.

\subsection{Relations between the function $f$ and the corresponding operator $F$}

\begin{restatable}{lemma}{obviousLemma}
\label{lem:obvious}
For any convex-concave function \( f \), a min-max sub-operator \( F \), and points \( (x_a, y_a) \in \cx\times\cy, (x_b, y_b) \in \cx\times\cy \), we have  
\begin{equation}\label{eqn:homo}
    \langle F(x_a, y_a), (x_a, y_a) - (x_b, y_b) \rangle \overset{(a)}{\geq} f(x_a, y_b) - f(x_b, y_a) \overset{(b)}{\geq} \langle F(x_b, y_b), (x_a, y_a) - (x_b, y_b) \rangle.
\end{equation}
\end{restatable}

\begin{proof}
Since the function $f$ is convex-concave the points $z_a,z_b$ satisfy,
\begin{align}
f(x_b,y_a)-f(x_a,y_a)-\ang{\nabla_x f(x_a,y_a),x_b-x_a}\geq 0\\\label{eqn:ccset6}
f(x_a,y_a)-f(x_a,y_b)-\ang{\nabla_y f(x_a,y_a),y_a-y_b}\geq 0
\end{align}
adding the above two equations and rearranging we obtain the left inequality (a). Furthermore the following also hold,
\begin{align*}
f(x_a,y_b)-f(x_b,y_b)-\ang{\nabla_x f(x_b,y_b),x_ 1-x_b}\geq 0\\ \label{eqn:ccset14}
f(x_b,y_b)-f(x_b,y_a)-\ang{\nabla_y f(x_b,y_b),y_b-y_a}\geq 0
\end{align*}
adding the above two equations and rearranging we obtain the right inequality (b).
\end{proof}

We now provide a similar bound for the \ec~function class.

\begin{restatable}{lemma}{ecTwoLemma}\label{lem:ecec2}

A \ec~function \( f(x,y) \) with parameter $\alpha$ and bounded operator norm \( \max_{z\in \mathcal{Z}} \|F(x,y)\| \leq L_0 \) over a domain with diameter \( \max_{(x_a, y_a),(x_b, y_b)\in \cx\times\cy}\|(x_a, y_a)-(x_b, y_b)\| \leq D \) satisfies the following:
\[
f(x_a,y_b)-f(x_b,y_a) \geq  \langle F(x_b,y_b),(x_a,y_a)-(x_b,y_b) \rangle + \frac{\gamma}{2} ((x_a-x_b,y_a-y_b)^\top M(x_b,y_b)(x_a-x_b,y_a-y_b).
\]
where \( \gamma = \frac{1}{2}\min\{\frac{1}{L_0D},\alpha\} \) and the matrix:
\[
M(x,y) = \begin{bmatrix} 
\nabla_x f(x, y) \nabla_x f(x, y)^\top & 0 \\ 
0 & \nabla_y f(x, y) \nabla_y f(x, y)^\top 
\end{bmatrix}.
\]
\end{restatable}

\begin{proof}
    From the definition \ref{def:expc} of exponentially concave functions we have that $g_{y_c}(x) = f(x,y_c)$ satisfies,
    $$g_{y_c}(x_b)-g_{y_c}(x_a)-\ang{g_{y_c}(x_a),x_b-x_a} \geq \frac{\gamma}{2} (x_a-x_b)^\top\nabla_x g(x_b) g(x_b)^\top(x_a-x_b) $$
which gives,
\begin{equation}\label{eqn:ec21}
    f(x_a,y_b)-f(x_b,y_b)-\ang{\nabla_x f(x_b,y_b),x_a-x_b)} \geq \frac{\gamma}{2} (x_a-x_b)^\top\nabla_x f(x_b,y_b) \nabla_xf(x_b,y_b)^\top(x_a-x_b)
\end{equation}

    and $g_{x_c}(y) = f(x_c,y)$ satisfies,
$$-g_{x_c}(y_b)+g_{x_c}(y_a)+\ang{g_{x_c}(y_a),y_b-y_a} \geq \frac{\gamma}{2} (y_a-y_b)^\top \nabla_y g(y_b) \nabla_y g(y_b)^\top(y_a-y_b)$$
 since $-f(x_c,y)$ is exponentially concave. This gives,
\begin{equation}\label{eqn:ec22}
    f(x_b,y_b)-f(x_b,y_a)-\ang{\nabla_y f(x_b,y_b),y_b-y_a)} \geq \frac{\gamma}{2} (y_a-y_b)^\top \nabla_y f(x_b,y_b) \nabla_y f(x_b,y_b)^\top(y_a-y_b)
\end{equation}
adding the equations Eq.~\eqref{eqn:ec21} and Eq.~\eqref{eqn:ec22} and rearranging we obtain the statement of the lemma.
\end{proof} 

As a corollary we restate the above fact purely in terms of the operator $F(x,y)$ and the matrix $M(x,y)$.

\begin{restatable}{corollary}{opecTwoCorollary}
\label{cor:opec2}  
The operator associated with a \ec~function \( f(x,y) \) with bounded operator \( \max_{(x, y)\in \cx\times\cy} \|F(x,y)\| \leq L_0 \) over a domain with diameter \( \max_{(x_a, y_a),(x_b, y_b)\in \cx\times\cy}\|(x_a, y_a)-(x_b, y_b)\| \leq D \) satisfies $\langle F(x_a, y_a) - F(x_b, y_b), (x_a, y_a) - (x_b, y_b) \rangle \geq \frac{\gamma}{2} ((x_a, y_a) - (x_b, y_b))^\top M(x_b, y_b) ((x_a, y_a) - (x_b, y_b)), \quad \forall (x_a, y_a), (x_b, y_b) \in \cx\times\cy
$ and also,
\[
\frac{\nabla F(x,y) + \nabla F(x,y)^\top}{2} \geq \frac{\gamma}{2} M(x,y), \quad \forall (x, y) \in \cx\times\cy.
\]
\end{restatable}

\begin{proof}

    From lemma \ref{lem:ecec2} we have, 
    $$f(x_a,y_b)-f(x_b,y_a) \geq  \ang{F(z_b),z_a-z_b}+\frac{\gamma}{2} (z_a-z_b)^\top M(z_b) (z_a-z_b).$$
    But we have from Eq.~\eqref{eqn:homo} that $$\ang{F(z_a),z_a-z_b} \geq f(x_a,y_b)-f(x_b,y_a).$$
    Thus we have,
    $$\ang{F(z_a),z_a-z_b} \geq \ang{F(z_b),z_a-z_b}+\frac{\gamma}{2} (z_a-z_b)^\top M(z_b) (z_a-z_b),$$
    rearranging we obtain the first statement of the lemma.
    Setting $z_a=z$ and $z_b = z+\tau \delta z$ in the first statement we obtain,
        $$\ang{F(z+\tau \delta z)-F(z),\tau \delta z} \geq \frac{\gamma}{2} (\tau \delta z)^\top F(z+\tau \delta z)F(z+\tau \delta z)^\top (\tau  \delta z) $$
        dividing both sides by $\tau^2$ we obtain,

        $$\lim_{\tau \rightarrow 0}\ang{\frac{F(z+\tau \delta z)-F(z)}{\tau}, \delta z} \geq \lim_{\tau \rightarrow 0} \frac{\gamma}{2}  \delta z^\top F(z+\tau \delta z)F(z+\tau \delta z)^\top  \delta z_t.$$
        
        Since the limits of both the LHS and RHS exist and the limit preserves inequalities, applying the limit $\tau \rightarrow 0$ to both sides we obtain
        
        $$\lim_{\tau \rightarrow 0}\ang{\frac{F(z+\tau \delta z)-F(z)}{\tau}, \delta z} \geq \lim_{\tau \rightarrow 0} \frac{\gamma}{2}  \delta z^\top F(z+\tau \delta z)F(z+\tau \delta z)^\top  \delta z_t.$$

        since the LHS is $\ang{\delta z, \nabla F \delta z}$ we obtain the second statement of the lemma.
\end{proof}

As a final step along the same lines we present a bound on $f(x_a,y_b)-f(x_b,y_a)$ for strongly convex-strongly concave functions.

\begin{restatable}{theorem}{strstrTheorem}
\label{thm:strstr}
A strongly convex-strongly concave function $f(x,y)$ with parameter $\lambda$ satisfies the following:$$f(x_a,y_b)-f(x_b,y_a) \geq \ang{F(x_b,y_b),(x_a,y_a)-(x_b,y_b)}+\frac{\lambda}{2} \|(x_a,y_a)-(x_b,y_b)\|^2,$$
for any $(x_a,y_a),(x_b,y_b)\in \cx\times\cy$.
\end{restatable}
\begin{proof}

    Since $f(x,y)$ is strongly convex-strongly concave we have that $g_{y_c}(x) = f(x,y_c)$ satisfies,
    $$g_{y_c}(x_b)-g_{y_c}(x_a)-\ang{g_{y_c}(x_a),x_b-x_a} \geq \frac{\gamma}{2} \|x_a-x_b\|^2 $$
which gives,
\begin{equation}\label{eqn:ss21}
    f(x_a,y_b)-f(x_b,y_b)-\ang{\nabla_x f(x_b,y_b),x_a-x_b)} \geq \frac{\gamma}{2} \|x_a-x_b\|^2
\end{equation}

    and $g_{x_c}(y) = f(x_c,y)$ satisfies,
$$-g_{x_c}(y_b)+g_{x_c}(y_a)+\ang{g_{x_c}(y_a),y_b-y_a} \geq \frac{\gamma}{2} \|y_a-y_b\|^2$$
 since $-f(x_c,y)$ is strongly convex. This gives,
\begin{equation}\label{eqn:ss22}
    f(x_b,y_b)-f(x_b,y_a)-\ang{\nabla_y f(x_b,y_b),y_b-y_a)} \geq \frac{\gamma}{2} |y_a-y_b\|^2
\end{equation}
adding the equations Eq.~\eqref{eqn:ss21} and Eq.~\eqref{eqn:ss22} and rearranging we obtain the statement of the lemma.
\end{proof}

\subsection{Properties and relations between function classes}

In this section we present properties of the various function classes presented in the paper. We start by showing that strong convexity-strong concavity implies two-sided PL inequality and \ec.
\begin{theorem}
\label{thm:smt2combined}
A $\lambda$ strongly convex-strongly concave differentiable function satisfies the two-sided PL-inequality with parameter $\frac{1}{2\gamma}$. Furthermore, if its operator norm is bounded by $L_0$, then it satisfies the \ec~condition with parameter $\frac{\lambda}{L_0^2}$.
\end{theorem}
\textbf{Proof of Theorem \ref{thm:smt2combined} Part 1}

\begin{proof}
Consider a $\lambda$ strongly convex-strongly concave function $f(x,y)$. We know that a $\lambda$ strongly convex function satisfies the PL-inequality with parameter $\frac{1}{2\lambda}$ (\cite{karimi2016linear}). Thus we have that for any point $z_0 = (x_c,y_c)$ (since $f(x,y_c)$ and $-f(x_c,y)$ are strongly convex),

$$f(x,y_c)-\min_x f(x,y_c) \leq \frac{1}{2\lambda}\nabla_x \|f(x,y_c)\|^2$$
$$-f(x_c,y)-\min_y (-f(x_c,y)) \leq \frac{1}{2\lambda }\nabla_y \|f(x_c,y)\|^2$$

since the above hold for any $z_0\in \cz$ we have that the function $f$ satisfies the 2-sided PL inequality with $\mu_1=\mu_2=\frac{1}{2\lambda}$.
\end{proof}

Before we provide the proof for the second part of Theorem \ref{thm:smt2combined}, we provide the following Lemma.

\begin{lemma}
     A strongly convex-strongly concave functions satisfies the following,
    $$f(x_a,y_b)-f(x_b,y_a) \geq  \ang{F(z_b),z_a-z_b}+\frac{\gamma}{2} (z_a-z_b)^\top M (z_a-z_b).$$
\end{lemma}
\begin{proof}

    Since $f(x,y)$ is strongly convex-strongly concave we have that $g_{y_c}(x) = f(x,y_c)$ satisfies,
    $$g_{y_c}(x_b)-g_{y_c}(x_a)-\ang{g_{y_c}(x_a),x_b-x_a} \geq \frac{\gamma}{2} \|x_a-x_b\|^2 $$
which gives,
\begin{equation}\label{eqn:ss21a}
    f(x_a,y_b)-f(x_b,y_b)-\ang{\nabla_x f(x_b,y_b),x_a-x_b)} \geq \frac{\gamma}{2} \|x_a-x_b\|^2
\end{equation}

    and $g_{x_c}(y) = f(x_c,y)$ satisfies,
$$-g_{x_c}(y_b)+g_{x_c}(y_a)+\ang{g_{x_c}(y_a),y_b-y_a} \geq \frac{\gamma}{2} \|y_a-y_b\|^2$$
 since $-f(x_c,y)$ is exponentially concave. This gives,
\begin{equation}\label{eqn:ss22a}
    f(x_b,y_b)-f(x_b,y_a)-\ang{\nabla_y f(x_b,y_b),y_b-y_a)} \geq \frac{\gamma}{2} |y_a-y_b\|^2
\end{equation}
adding the equations Eq.~\eqref{eqn:ss21a} and Eq.~\eqref{eqn:ss22a} and rearranging we obtain the statement of the lemma.
\end{proof} 
\textbf{Proof of Theorem \ref{thm:smt2combined} Part 2}

\begin{proof}

We know the operator $F$ of a $\lambda$ strongly convex-strongly concave function $f(x,y)$ with bounded operator norm $\|F(x,y)\| \leq L_0$ satisfies,
$$\ang{F(z_a)-F(z_b),z_a-z_b} \geq {\lambda} \|z_a-z_b\|^2 \geq \frac{\lambda}{L_0^2} (z_a-z_b)^\top M(z_b) (z_a-z_b)$$
since for any two matrices $x^\top A x \geq x^\top B x~\forall~x$ if $\lambda_{min}(A) \geq \lambda_{max}(B)$, setting $A=I$ and $B = \frac{M(z_b)}{L_0^2}$. Following the proof of Theorem \ref{thm:gprime-combined} we have that the function $f$ is $\frac{\lambda}{L_0^2}$ exponentially concave.
\end{proof}

We now present a result describing properties of the static gap function $g'_t(x,y) = f(x,y')-f(x',y)$ conditioned on the properties of $f$. These results play a crucial role in establishing the \sdg~and \dsp~rates for these functions.
\begin{theorem}
\label{thm:gprime-combined}
Let $g'(x,y) = f(x,y') - f(x',y)$ be defined with respect to a function $f(x,y)$ where $(x',y')$ is a fixed point $(x',y')\in \cx\times\cy$. Then:
\begin{itemize}
    \item If the function $f(x,y)$ is \ec~with parameter $\alpha$ over a domain with diameter,\newline
$\max_{(x_a,y_b),(x_b,y_b)\in \cx\times\cy}\|(x_a,y_a)-(x_b,y_b)\| \leq D$ and has bounded operator $\max_{z\in \cz} \|F(x,y)\| \leq G$, then $g'(x,y)$ is $\frac{\gamma}{4}$ exponentially concave where $\gamma = \frac{1}{2}\min\left\{\frac{1}{L_0D},\alpha\right\}$.

    \item If the function $f(x,y)$ is $\lambda$ strongly convex-strongly concave, then $g'(x,y)$ is $\lambda$ strongly convex in $(x,y)$.
\end{itemize}
\end{theorem}
\textbf{Proof of Theorem \ref{thm:gprime-combined} Part 1}

\begin{proof}

    Since $f(x,y)$ is \ec~we have $h(x) = f(x,y')$ is exponentially concave and $g'(y) = -f(x',y)$ is exponentially concave. Consider two points $z_a = (x_a,y_a),z_b=(x_b,y_b)$. Thus we have, $$h(x_a)-h(x_b)-\nabla h(x_b)^\top (x_a-x_b) \geq \frac{\gamma}{2} (x_a-x_b)^\top(\nabla_x h(x_b) \nabla_x h(x_b)^\top)(x_a-x_b)$$
    and also, 
    $$g(y_a)-g(y_b)-\nabla g(y_b)^\top (y_a-y_b) \geq \frac{\gamma}{2} (y_a-y_b)^\top(\nabla_y g(y_b) \nabla_y g(y_b)^\top)(y_a-y_b)$$
    where $\gamma=\frac{1}{2}\min\{\frac{1}{L_0D},\alpha\}$. Adding the two and using the fact that,
    \begin{align*}
        2((x_a-x_b)^\top\nabla_x h(x_b) \nabla_x h(x_b)^\top(x_a-x_b)&+(y_a-y_b)^\top\nabla_y g(y_b) \nabla_y g(y_b)^\top(y_a-y_b))\\
        &\geq (x_a-x_b)^\top\nabla_x h(x_b) \nabla_x h(x_b)^\top(x_a-x_b)\nonumber\\
        &+(y_a-y_b)^\top\nabla_y g(y_b) \nabla_y g(y_b)^\top(y_a-y_b)\nonumber\\
        &+2(x_a-x_b)^\top \nabla_x h(x_b) \nabla g(y_b)^\top (y_a-y_b)\nonumber
    \end{align*}
    we obtain,    
    \begin{align}\label{eq:ecmain}
        g'(z_a)-g'(z_b)-\nabla_z g'(z_b)^\top (z_a-z_b) \geq \frac{\gamma}{4} (z_a-z_b)^\top \nabla_z g'(z) \nabla_z g'(z)^\top (z_a-z_b)
    \end{align}
    We know $|\frac{\gamma}{4} \nabla_z g(z_b)^\top (z_a-z_b)| \leq L_0D\gamma \leq \frac{1}{8}$ since $\frac{\gamma}{4}\leq \frac{L_0D}{8}$. We have $a+a^2\geq -\ln(1-a)$ for $|a| \leq \frac{1}{8}$. Thus plugging in $a= \frac{\gamma}{4}\nabla_z g'(z_b)^\top (z_a-z_b)$ we have from \eqref{eq:ecmain},
    $$\nabla_z g'(z_b)^\top (z_a-z_b) + \frac{\gamma}{4} (z_a-z_b)^\top \nabla_z g'(z) \nabla_z g'(z)^\top (z_a-z_b) \geq -\frac{\ln(1-\frac{\gamma \nabla_z g'(z_b)^\top (z_a-z_b)}{4})}{\frac{\gamma}{4}}.$$
    Plugging back in \eqref{eq:ecmain} we obtain, $g'(z_a)-g'(z_b)\geq -\frac{\ln(1-\frac{\gamma \nabla_z g'(z_b)^\top (z_a-z_b)}{4})}{\frac{\gamma}{4}}.$ Raising both sides of the function to the power of exponential we have (since the exponential function is monotone),
    $$e^{-\frac{\gamma}{4}g'(z_b)} (1-\frac{\gamma \nabla_z g'(z_b)^\top (z_a-z_b)}{4})\geq e^{-\frac{\gamma}{4} g'(z_a)}$$
     Consider the function $p(z) = e^{-\frac{\gamma}{4}g'(z)}$ with gradient $\nabla_z p(z) = \frac{\gamma}{4} e^{-\frac{\gamma}{4}g'(z)} \nabla_z g'(z)$. The above implies, 
     $$p(z_b)(1-\frac{\nabla_z p(z_b)^\top (z_a-z_b)}{p(z_b)}) \geq p(z_a)$$
     which upon rearranging gives, $p(z_b)-p(z_a) \geq \nabla_z p(z_b)^\top (z_a-z_b)$. This implies that the function $p(z)$ is concave and the function $g'(z)$ is $\frac{\gamma}{4}$ exponentially concave.  
\end{proof}

\textbf{Proof of Theorem \ref{thm:gprime-combined} Part 2}

\begin{proof}

    Consider two points $z_a = (x_a,y_a)$ and $z_b = (x_b,y_b)$. If the function $f$ is $\lambda$-strongly convex-strongly concave, plugging in $(x_c,y_c) = (x',y') \in \cz$ in definition, \ref{def:strcon} for the point $z'=(x',y')\in \cz$ we have, $f(x_b,y') \geq f(x_a,y')+\nabla_x f(x_a,y')+\frac{\lambda}{2} \|x_b-x_b\|^2~\forall~x_a,x_b\in \cx$ and $f(x',y_a) \geq f(x',y_b)-\nabla_y f(x',y_a)+\frac{\lambda}{2} \|y_b-y_a\|^2~\forall~y_a,y_b\in \cy$. Adding the two we have,
    \begin{align*}
        f(x_b, y') - f(x', y_b) &\geq f(x_a, y') - f(x', y_a)+(\nabla_x f(x_a, y') - \nabla_y f(x', y_a)) \\
        &\quad + \frac{\lambda}{2} \big( \|x_b - x_a\|^2 + \|y_b - y_a\|^2 \big), \forall x_a, x_b \in \mathcal{X}, \; y_a, y_b \in \mathcal{Y}.
    \end{align*}
    Since $\nabla_z g'(z)|_{z=z_a} = (\nabla_x f(x_a, y') - \nabla_y f(x', y_a))$ we have that $g'(z)$ is strongly convex with the parameter $\lambda$.
\end{proof}
\section{Proofs for OMMO}\label{app:sec2}

Before we present the proofs for OGDA and OMMNS in the strongly convex-strongly concave and \ec~setting respectively, we present a meta-algorithm that has OGDA and OMMNS as a special case and a meta result for the same. We obtain bounds on \dsp~with this algorithm as the base learner in the IFLH framework. 

\begin{algorithm}
\caption{Lower Regular Algorithms (LRAs)}
\label{alg:lraalg}
\KwIn{$A_0$,$(x_0,y_0), (u_0^x,u_0^y) \in \mathcal{X} \times \mathcal{Y}$}
\For{$t = 1$ \textbf{to} $T$}{
    Play $(x_t,y_t)$, observe cost $f_t(x_t,y_t)$ and $F(x_t,y_t)$\\
    Regularity update: $A_t = A_{t-1} + M_t(x_t,y_t)$\\
    $(u_{t+1}^x,u_{t+1}^y) = (x_t,y_t) - \frac{1}{\gamma} A_t^{-1} F(x_t,y_t)$\\
    $(x_{t+1},y_{t+1}) = \argmin_{(x,y) \in \cx\times\cy} \|(u_{t+1}^x,u_{t+1}^y) - (x,y)\|_{A_t}^2$\
}
\end{algorithm}

\begin{lemma}\label{lem:lra}
    The iterates action pairs $z_t=(x_t,y_t)$ of the LRA algorithm satisfy
    \begin{align}\label{final2}
    \sum_{t=1}^TF_t(z_t)^\top (z_t-z') &\leq \frac{1}{2\gamma}\sum_{t=1}^T F_t(z_t)^\top A_t^{-1}F_t(z_t)+\frac{\gamma}{2}\sum_{t=1}^T(z_t-z')^\top M_t(z_t)(z_t-z')\\\nonumber
    &\quad\quad+\frac{\gamma}{2}(z_a-z')^\top (A_1-F_1(z_a)F_1(z_a)^\top) (z_a-z')~\forall~z'=(x',y')\in \cx\times \cy
 \end{align}
\end{lemma}
\begin{proof}

         Define $z' = \argmin_{z}\sum_{t=1}^T g'_t(z)$. From the update rule we have, $u_{t+1}-z' = z_t-z'-\frac{1}{\gamma}A_t^{-1} F_t(z_t)$ thus, $A_t(u_{t+1}-z') = A_t(z_t-z')-\frac{1}{\gamma}F_t(z_t)$.
    Combining the previous two we obtain,
    $$(u_{t+1}-z')^\top A_t(u_{t+1}-z') = (z_t-z')^\top A_t(z_t-z')-\frac{2}{\gamma}F_t(z_t)^\top(z_t-z')+\frac{1}{\gamma^2}F_t^\top(z_t)(A_t)^{-1}F_t(z_t)$$
    Since $z_{t+1}$ is a projection of $u_{t+1}$ in the norm induced by $A_t$ (which is PSD) we have,
    \begin{align}
        (u_{t+1}-z')^\top A_t(u_{t+1}-z') &= \|u_{t+1}-z'\|_{A_t}^2\nonumber\\
        &\geq \|z_{t+1}-z'\|_{A_t}^2\nonumber\\
        &=(z_{t+1}-z')^\top A_t (z_{t+1}-z')\nonumber
    \end{align}
     Overall we have,
\begin{equation}\label{final}
F_t(z_t)^\top (z_t-z') \leq \frac{1}{2\gamma}F_t^\top A_t^{-1}F_t(z_t)+\frac{\gamma}{2}(z_{t}-z')^\top A_t(z_{t}-z')-\frac{\gamma}{2}(z_{t+1}-z')^\top A_t (z_{t+1}-z')    
\end{equation}
Now, summing up over $t=1$ to $T$ we get that,
\begin{align*}
    \sum_{t=1}^TF_t(z_t)^\top (z_t-z') &\leq \frac{1}{2\gamma}\sum_{t=1}^T F_t(z_t)^\top A_t^{-1}F_t(z_t)+\frac{\gamma}{2}(z_{1}-z')^\top A_1 (z_{1}-z')\\
    &\quad\quad+\frac{\gamma}{2}\sum_{t=2}^T(z_t-z')^\top(A_t-A_{t-1}) (z_t-z')-\frac{\gamma}{2}(z_{T+1}-z')^\top A_T (z_{T+1}-z')\\\nonumber
    &= \frac{1}{2\gamma}\sum_{t=1}^T F_t(z_t)^\top A_t^{-1}F_t(z_t)+\frac{\gamma}{2}\sum_{t=1}^T(z_t-z')^\top M(z_t)(z_t-z')\\\nonumber
    &\quad\quad+\frac{\gamma}{2}(z_a-z')^\top (A_1-F_1(z_a)F_1(z_a)^\top) (z_a-z')
 \end{align*}

 where the final equality follows from the fact that $A_t-A_{t-1} = M(z_t)$ and $A_T$ is positive semi-definite.
\end{proof}

\subsection{Static Regret Proofs}\label{app:static}
\OGDATheorem*
\begin{proof}
From Lemma \ref{lem:lra} setting $M_t(z_t)=I$ we have,
\begin{align*}
    \sum_{t=1}^TF_t(z_t)^\top (z_t-z') &\leq \frac{1}{2\gamma}\sum_{t=1}^T F_t(z_t)^\top A_t^{-1}F_t(z_t)+\frac{\gamma}{2}(z_{1}-z')^\top A_1 (z_{1}-z')\\
    &\quad\quad+\frac{\gamma}{2}\sum_{t=2}^T(z_t-z')^\top(A_t-A_{t-1}) (z_t-z')\\\nonumber
    &\leq \frac{1}{2\gamma}\sum_{t=1}^T F_t(z_t)^\top A_t^{-1}F_t(z_t)+\frac{\gamma}{2}\sum_{t=1}^T(z_t-z')^\top(z_t-z')\\\nonumber
    &\quad\quad+\frac{\gamma}{2}(z_a-z')^\top A_1 (z_a-z')
 \end{align*}
Since we know $f$ is strongly convex-strongly concave, from Theorem \ref{thm:strstr} we have,
    $$g'_t(x_t,y_t) = f_t(x_t,y')-f_t(x',y_t) \leq \ang{F_t(z_t),z_t-z'}-\frac{\lambda}{2} \|z_t-z'\|^2$$
which along with the fact that $A_t^{-1} = \frac{1}{t}I$ gives,
\begin{align}\label{final2}
    \sum_{t=1}^T f_t(x_t,y')-f_t(x',y_t) \leq \frac{1}{2\lambda}\sum_{t=1}^T \frac{1}{t}\|F_t(z_t)\|^2
 \end{align}
 since $\sum_{t=1}^T \frac{1}{t} = \log(T)+1$ and $\|F(z)\| \leq L_0~\forall~z\in \mathcal{Z}$ we have that, 
 \begin{align}\label{final2}
    \textit{SDual-Gap}_T = \sum_{t=1}^T f_t(x_t,y')-f_t(x',y_t) \leq \frac{L_0^2}{2\lambda}(\log T+1) \leq \frac{L_0^2}{\lambda}\log T 
 \end{align}
\end{proof}

Before we present the regret bounds for the OMMNS algorithm in the \ec~setting, we present proof that the function class in Example \ref{ex:log} is \ec.

\logexample*
\begin{proof}
 
a) Clearly $h(x,y) = e^{f(x,y)} = \sum_i \frac{a_ix_i}{y_i}$ is linear in $x$ and thus concave in $x$. Thus $g_{y_c}(x)=f(x,y_c)$ is 1-exponentially concave $\forall y_c = (y_1,y_2,...,y_n) \in \cy$.

b) For $h(x,y) = e^{-f(x,y)} = \frac{1}{\sum_i \frac{a_ix_i}{y_i}}$. Let $g(y) = \sum_i \frac{a_ix_i}{y_i}$.

$$\nabla_y h(x,y) = \frac{1}{g(y)^2} \cdot (\frac{a_1x_1}{y_1^2},\frac{a_2x_2}{y_2^2},...,\frac{a_nx_n}{y_n^2}) \defeq \frac{u}{g(y)^2} .$$

The Hessian matrix is then,

$$\nabla^2_y h(x,y) = \frac{2}{g(y)^3}(uu^\top)-\frac{2}{g(y)^2}diag(d_1,d_2,...,d_n)$$

where $d_i = \frac{a_ix_i}{y_i^3}$. Evaluating $v^\top \nabla^2h(y)v$ we obtain,

\[
v^\top \nabla^2 h(y) v = \frac{2}{g(y)^3}
\left[
\left(\sum_{i=1}^n \frac{a_i x_i}{y_i^2} v_i\right)^2-g(y)
\sum_{i=1}^n \frac{a_i x_i}{y_i^3} v_i^2
\right].
\]

Consider the fact that $\left(\sum_{i=1}^n \frac{a_i x_i}{y_i^2} v_i\right)^2$ can be written as $\left(\sum_{i=1}^n \sqrt{\frac{a_i x_i}{y_i}} \frac{v_i}{y_i}\sqrt{\frac{a_ix_i}{y_i}}\right)^2$. We then have from Cauchy-Schwarz,

$$\left(\sum_{i=1}^n \sqrt{\frac{a_i x_i}{y_i}} (\frac{v_i}{y_i}\sqrt{\frac{a_ix_i}{y_i}})\right)^2 \leq (\sum_i^n \frac{a_ix_i}{y_i})(\sum_{i=1}^n \frac{a_ix_iv_i^2}{y_i^3}).$$

Thus $v^\top \nabla^2 h(y) v \leq 0$ and the function $e^{-f(x,y)}$ is concave in $y$ and $g_{x_c}(y) = -f(x_c,y)$ is 1-exponentially concave $\forall x_c =(x_a,x_b)\in \cx$. Thus, overall we have that $f(x,y)$ is \ec.

\end{proof}

\staticonsTheorem*
\begin{proof}

From Lemma \ref{lem:lra} we have
\begin{align*}
    \sum_{t=1}^TF_t(z_t)^\top (z_t-z') &\leq \frac{1}{2\gamma}\sum_{t=1}^T F_t(z_t)^\top A_t^{-1}F_t(z_t)+\frac{\gamma}{2}(z_{1}-z')^\top A_1 (z_{1}-z')\\
    &\quad\quad+\frac{\gamma}{2}\sum_{t=2}^T(z_t-z')^\top(A_t-A_{t-1}) (z_t-z')\\\nonumber
    &\leq \frac{1}{2\gamma}\sum_{t=1}^T F_t(z_t)^\top A_t^{-1}F_t(z_t)+\frac{\gamma}{2}\sum_{t=1}^T(z_t-z')^\top M(z_t)(z_t-z')\\\nonumber
    &\quad\quad+\frac{\gamma}{2}(z_a-z')^\top (A_1) (z_a-z')
 \end{align*}
    In the last inequality we use the fact that $A_t-A_{t-1} = M(z_t)$, and the fact that the matrix $A_T$ is PSD and hence the last term before the inequality is negative. Since the functions $f_t$ are \ec, $z_a = z'$ and $z_b = z_t$ in Lemma \ref{lem:ecec2} we have that
    $$g'(x_t,y_t) = f(x_t,y')-f(x',y_t) \leq \ang{F_t(z_t),z_t-z'}-\frac{\gamma}{2} (z_t-z')^\top M_t(z_t) (z_t-z'), $$
    where $M_t(z_t) = F(z_t)F(z_t)^\top$.
    Then we have
    $$\sum_{t=1}^T f_t(x_t,y')-f(x',y_t) \leq \frac{1}{2\gamma} \sum_{t=1}^T F_t(z_t)^\top A_t^{-1}F_t(z_t)+\frac{\gamma}{2}(z_a-z')^\top (A_1-F_1(z_a)F_1(z_a)^\top) (z_a-z')$$
Using the algorithm parameters $A_1-M(z_t) = \epsilon I_d$, $\epsilon = \frac{1}{\gamma^2D^2}$ and the diameter $\|z_a-z'\|^2 \leq D^2$ we have,
\begin{align}\label{eqn:omd}
  \text{SDual-Gap}_T  &\leq \sum_{t=1}^T g_t(z_t)\leq \sum_{t=1}^TR_t \leq \frac{1}{2\gamma} \sum_{t=1}^T F_t(z_t)A_t^{-1}F_t(z_t)+\frac{\gamma}{2}D^2\epsilon\\
    &\leq \frac{1}{2\gamma} \sum_{t=1}^T F_t(z_t) A_t^{-1} F_t(z_t)+\frac{1}{2\gamma}\nonumber
\end{align}
Since $\gamma = \frac{1}{2}\min\{\frac{1}{L_0D},\alpha\}$, we have $\frac{1}{\gamma} \leq 2(\frac{1}{\alpha}+L_0D)$.

Finally we have, 
$$ \sum_{t=1}^T F_t(z_t)^\top A_t^{-1} F_t(z_t) = \sum_{t=1}^T A_t^{-1} F_t(z_t) F_t(z_t)^\top = \sum_{t=1}^T (A_t^{-1})(A_t-A_{t-1}) \leq \sum_{t=1}^T 
\log \frac{|A_t|}{|A_{t-1}|} = \log \frac{|A_T|}{|A_{0}|} $$
Where the first equality follow from the fact that $a^\top b = Tr(ab^\top)$ and the last inequality follows from Lemma 4.7 in \cite{hazan2016introduction} and from the fact that $A_t,A_{t-1}$ are PSD. Since $A_T = \sum_{t=1}^T F_t(z_t)F_t(z_t)^\top + \varepsilon I_d$ and $\|F_t(z_t)\| \leq L_0$, the largest eigenvalue of $A_T$ is at most $TL_0^2 + \varepsilon$. Hence the determinant of $A_T$ can be bounded by $|A_T| \leq (TL_0^2 + \varepsilon)^d$ (since $A_T \in \mathbb{R}^{d\times d}$). Thus since, $\varepsilon = 1$ and $\gamma = \frac{1}{2}\min\{\frac{1}{L_0D}, \alpha\}$, for $T > 4$,
\begin{align*}
\sum_{t=1}^T F_t(z_t)^\top A_T^{-1}F_t(z_t) 
&\leq d\log(TL_0^2\gamma^2D^2 + 1) \\
&\leq d\log T.
\end{align*}

Plugging into Eq.~\eqref{eqn:omd} we obtain,
\begin{equation*}
\text{SDual-Gap}_T \leq (\frac{1}{\alpha} + L_0D)(d\log T + 1),
\end{equation*}
which implies the theorem for $d \geq 1,T\geq 3$.
\end{proof}

\subsubsection{\dyne~example}\label{app:dyne}

In this sub-section we compute the time average of the strategies for a low \dyne algorithm in Example \ref{ex:dyne} and show that the time average maintains a $\Theta(1)$ distance from the cumulative saddle-point, even though \dyne $\leq 1$. We first calculate the \dyne,

\begin{align}
    |\sum_{t=1}^T f_t(x_t,y_t)-\sum_{t=1}^T \min_x \max_y f_t(x,y)| &= |\sum_{t=1}^T f_t(x_t,y_t)|\nonumber\\
              &= |\sum_{t=1}^T (\frac{(-1)^t+1}{2}(2\mathbf{1}_{\{x_t < 0\}}-1))^2\nonumber\\
              &\quad\quad-(\frac{(-1)^t-1}{2}(2\mathbf{1}_{\{x_t < 0\}}-1))^2|\nonumber\\
              &=|\sum_{t=1}^T (-1)^t|\leq 1
\end{align}

We now calculate the distance of the average iterates from the cumulative saddle-point.

\begin{align}
    \left(\frac{\sum_t (x_t^* - x_t)}{T}, \frac{\sum_t (y_t^* - y_t)}{T}\right) &= (\sum_{t=1}^T (\frac{(-1)^t+1}{2}(2\mathbf{1}_{\{x_t < 0\}}-1)),-\sum_{t=1}^T (\frac{(-1)^t-1}{2}(2\mathbf{1}_{\{x_t < 0\}}-1)))\\
&= \left(\frac{\sum_{t \in \text{even}[1,T]} \big(2\mathbf{1}_{\{x_t < 0\}} - 1\big)}{T},
\frac{\sum_{t \in \text{odd}[1,T]} \big(2\mathbf{1}_{\{y_t < 0\}} - 1\big)}{T}\right)
\end{align}

The distance, $\left(|\frac{\sum_t (x_t^* - x_t)}{T}|, |\frac{\sum_t (y_t^* - y_t)}{T}|\right)$ can thus be $\Theta(1)$ if the actions are chosen such that $x_t,y_t<0~\forall~t$ or $x_t,y_t>0~\forall~t$ even when \dyne$\leq1$.
\subsection{Dynamic Regret Proofs}\label{app:dynamic}

We now discuss the theorems detailing the \dsp~of OGDA and \ommns. We will use OGDA and \ommns~as base learners in min-max Follow the Leader History (MMFLH), a modified version of IFLH~\citep{zhang2018dynamic} that is adapted to the min-max setting, which for clarity we present in full in  Algorithm~\ref{alg:iflh}. MMFLH uses the static gap function $g'(x,y)=f(x,y')-f(x',y)$ instead of $f(x)$ in the IFLH framework. For this setting we assume that the cumulative saddle-point is provided to the meta-framework in advance. In order to bound the \dsp, we first bound the \sasp~and then use the relation between \sasp~and \dsp~to bound it.  The following theorems, Theorem \ref{finalm1} and Theorem \ref{finalm2} begin with first extending the static regret results on \sdg~for the base learners to a static-regret like quantity over any time interval $[r,s]$ when the base learners are used in MMFLH.

\begin{algorithm}[H]
\caption{Min-Max Following the Leading History (MMFLH)}
\label{alg:iflh}
\DontPrintSemicolon

\KwIn{An integer $K$, Cumulative saddle-point $(x',y')$}
Initialize $S_0 = \emptyset$\;

\For{$t = 1,\ldots,T$}{
    Set $Z_t = 0$\;

    \For{$E^j \in S_{t-1}$}{
        \If{$e^j \le t$}{
            Update $S_{t-1} \leftarrow S_{t-1} \setminus \{E^j\}$\;
        }
        \Else{
            Set $Z_t = Z_t + \tilde{p}_t^j$\;
        }
    }

    \For{$E^j \in S_{t-1}$}{
        Set $p_t^j = \frac{\tilde{p}_t^j}{Z_t}\!\left(1-\frac{1}{t}\right)$\;
    }

    Set $S_t = S_{t-1} \cup \{E^t\}$\;
    Compute the ending time $e^t = E_K(t)$ according to Definition~3 in \citep{zhang2018dynamic} and set $p_t^t = \frac{1}{t}$\;

    Submit the solution
    \[
        (\mathbf{x}_t,\mathbf{y}_t) = \sum_{E^j \in S_t} p_t^j (\mathbf{x}_t^j,\mathbf{y}_t^j)
    \]
    and suffer loss $f_t(\mathbf{x}_t,y')-f(x',\mathbf{y}_t)$\;

    Set $Z_{t+1} = 0$\;
    \For{$E^j \in S_t$}{
        Compute $p_{t+1}^j = \tilde{p}_t^j \exp\!\big(-\alpha_t (f_t(\mathbf{x}_t^j,y')-f_t(x',\mathbf{y}_t^j))\big)$\;
        Set $Z_{t+1} = Z_{t+1} + p_{t+1}^j$\;
        Pass the function $f_t(\cdot)$ to $E^j$\;
    }

    \For{$E^j \in S_t$}{
        Set $\tilde{p}_{t+1}^j = \frac{p_{t+1}^j}{Z_{t+1}}$\;
    }
}
\end{algorithm}

We now present a min-max extension of  Theorem 2 (\cite{zhang2018dynamic}), presenting the \sasp~for MMFLH when the OMMNS algorithm is used in it as a base learner. The proof uses the \sdg~obtained by OMMNS (Theorem \ref{thm:staticons}) and the min-max exponential concavity of the functions $\{f_t\}_{t=1}^T$.

\begin{theorem}\label{finalm1}
     If $f_t$ are \ec~then upon using \ommns~as the base learner in (MMFLH) we have for any $[r,s]\subseteq T$ and $m\leq 
    \log _K(s-r+1)+1$, 
    $$\sum_{t=r}^sg'_t(z_t)-\min_{z\in \cx\times \cy} \sum_{t=r}^s g'_t(z) \leq (\frac{(2d+1)m+2}{\alpha}+2dmL_0D)\log  T$$
    Also,
    $$\text{SASP-Regret}_T(\tau) \leq (\frac{(2d+1)\bar m+2}{\alpha}+2d\bar m L_0D)\log  T = O(\frac{d\log ^2 T}{\log  K})$$
    where $\bar{m} = \log _K\tau + 1$. Where $z_t=(x_t,y_t)$ is the action pair of the players at time $t$ and $z'=\argmin_{z\in \cx\times \cy} \sum_{t=r}^s g'_t(z)$ is the saddle point of the sum of functions $\sum_{t=1}^T f_t(x,y)$.
\end{theorem}
\begin{proof}

Our proof will follow the structure of that in \cite{zhang2018dynamic}. We first note that under the IFLH, since the function $g_t'(z_t) = f_t(x_t,y')-f_t(x',y_t) $ is $\alpha=\frac{\gamma}{4}$ exponentially concave, where $\gamma = \frac{1}{2}\min\{\frac{1}{L_0D},\alpha'\}$ when $f_t$ is \ec~with parameter $\alpha'$ from Lemma \ref{lem:ecec2}, we have that Claim 3.2 in \cite{hazan2009efficient} holds. Also, since Claim 3.3 also holds due to the MMFLH framework with $f_t(.)$ (in IFLH) substituted with $f_t(\cdot,y')-f(x',\cdot)$, Claim 3.1 holds following the proof given by \cite{hazan2009efficient}. Also, Lemma 1 of \cite{zhang2018dynamic} holds by the framework of MMFLH. 

From the second part of Lemma 1, we know that there exist \( m \) segments 
\( I_j = [t_j, e^{t_j} - 1] \), \( j \in [m] \) with \( m \leq \lceil \log _K(s - r + 1) \rceil + 1 \), such that 
\( t_1 = r \), \( e^{t_j} = t_{j+1} \), \( j \in [m-1] \), and \( e^{t_m} > s \).

We have from Claim 3.1 that the expert $E^{t_j}$ is alive during the period $[t_j,e^{t_j}-1]$.

$$\sum_{t_j}^{e^{t_j}-1}g'_t(z_t)-g'_t(z_t^{t_j}) \leq \frac{1}{\alpha}(\log  t_j+2\sum_{t_j+1}^{e^{t_j}-1}\frac{1}{t}),\forall~j\in [m-1].$$

where $z_{t_j}^{t_j},...,z_{e^{t_j}-1}^{t_j}$ is the sequence of action pairs generated by expert $E^{t_j}$. Similarly for the last segment, we have,

$$\sum_{t=t_m}^s g'_t(z_t)-g'_t(z_t^{t_m}) \leq \frac{1}{\alpha}(\log  t_m+2\sum_{t=t_m+1}^s\frac{1}{t}).$$

By adding things together, we have:

\begin{align}\label{eqn:cp1}
\sum_{j=1}^{m-1} \left( \sum_{t=t_j}^{e^{t_j - 1}} \big(g'_t(z_t) - g'_t(z_{t_j}) \big) \right) 
+ \sum_{t=t_m}^s \big(g'_t(z_t) - g'_t(z_{t_m}) \big) & \nonumber \\
\leq \frac{1}{\alpha} \sum_{j=1}^m \left( \log  t_j + \frac{2}{\alpha} \sum_{t=r+1}^s \frac{1}{t} \right) 
& \leq \frac{m+2}{\alpha} \log  T.
\end{align}

for any \( z \in \cx\times\cy \). From Theorem \ref{thm:staticons},
\begin{align}\label{eqn:par1}
    \sum_{t=t_j}^{e^{t_j - 1}} {g'_t(z_t^{t_j}) - g'_t(z)} \leq 2d ({L_0D+ \frac{1}{\alpha}})\log  T, \quad \forall j \in [m-1]
\end{align}
and
\begin{align}\label{eqn:par2}
\sum_{t=t_m}^s {g'_t(z_t) - g'_t(z)} \leq 2d(\frac{1}{\alpha} + L_0D)\log T.    
\end{align}

Combining equations \eqref{eqn:cp1}, \eqref{eqn:par1} and \eqref{eqn:par2} we have:
\[
\sum_{t=r}^s g'_t(z_t) - g'_t(z) \leq \left( \frac{(2d + 1)m+2}{\alpha} + 2dmL_0D \right) \log  T
\]
for any \( z \in \cx\times\cy \).

Recall $m\leq \lceil \log _K(s - r + 1) \rceil + 1$. This implies, $$\text{SASP-Reg}_T(\tau) \leq ( \frac{(2d + 1)\bar m+2}{\alpha} + 2d\bar mL_0D)\log T,$$ where $\bar m = \lceil \log _K(\tau) \rceil +1$.
\end{proof}

We now present a min-max extension of  Theorem 1 (\cite{zhang2018dynamic}), presenting the \sasp~for MMFLH when the OGDA algorithm is used in it as a base learner. 

\begin{theorem}\label{finalm2}
    If $f_t$ are $\lambda$ strongly convex-strongly concave then upon using OGDA~as the base learner in (MMFLH) we have for any $[r,s]\subseteq T$ and $m\leq 
    \log _K(s-r+1)+1$, 
    $$\sum_{t=r}^sg'_t(z_t)-\min_{z\in \cz} \sum_{t=r}^s g'_t(z) \leq \frac{L_0^2}{2\lambda}(m+(3m+4)\log  T)$$
    Also,
    $$\text{SASP-Regret}_T(\tau) \leq \frac{L_0^2}{2\lambda}(\bar m +(3\bar m +4)\log  T) = O(\frac{\log ^2 T}{\log  K})$$
    where $\bar{m} = \log _K\tau + 1$. Where $z_t=(x_t,y_t)$ is the action pair of the players at time $t$ and $z'=\argmin_{z\in \cx\times \cy} \sum_{t=r}^s g'_t(z)$ is the saddle point of the sum of functions $\sum_{t=1}^T f_t(x,y)$.
\end{theorem}
\begin{proof}

    From Lemma \ref{thm:smt2combined} we have that strongly convex-strongly concave functions are \ec~with $\gamma = \frac{\lambda}{L_0^2}$. Thus $g'$ is exponentially concave and following \ref{eqn:cp1} we have,

\begin{align}\label{eqn:ran0}
    \sum_{j=1}^{m-1} \left(\sum_{t=t_j}^{e^{t_j}-1}  g'_t(z_t) - g'_t(z_t^{t_j}) \right) + \sum_{t=t_m}^s \left( g'_t(z_t) - g'_t(z_t^{t_m}) \right)  \leq \frac{(m+2)L_0^2}{\lambda} \log  T. 
\end{align}

From Theorem \ref{thm:OGDA} for online gradient descent ascent we have, for any $z \in \cx\times\cy$,
\begin{align}\label{eqn:ran1}
    \sum_{t=t_j}^{e^{t_j} - 1} \left( g'_t(z_t^{t_j}) - g'_t(z) \right) \leq \frac{L_0^2}{2\lambda}(1 + \log  T), \quad \forall j \in [m-1].
\end{align}
and,
\begin{align}\label{eqn:ran2}
    \sum_{t=t_m}^{t} \left( g'_t(z_t^{t_m}) - g'_t(z) \right) \leq \frac{L_0^2}{2\lambda}(1 + \log  T). 
\end{align}
Finally, combining equations \eqref{eqn:ran0}, \eqref{eqn:ran1} and \eqref{eqn:ran2} we have,
\[
\sum_{t=r}^s\left( g'_t(z_t) - g'_t(z) \right) \leq \frac{L_0^2}{2\lambda} \left( m + (3m + 4) \log  T \right).
\]
for any $z\in \cx\times\cy$.

Recall $m\leq \lceil \log _K(s - r + 1) \rceil + 1$. This implies, $$\text{SASP-Reg}_T(\tau) \leq \frac{L_0^2}{2\lambda}(\bar m + (3\bar m + 4) \log  T ),$$ where $\bar m = \lceil \log _K(\tau) \rceil +1$.

\end{proof}

Corollaries 6 and 7 in \cite{zhang2018dynamic} use the relation between \sasp~and \dsp~to bound the latter. Since the \sasp~and \dsp~are simply the \sareg~and \dreg~with functions $f_t$ substituted with $g_t'$, we are also able to bound \dsp~using \sasp~in an analogous manner. The proof of the following theorem follows from corollaries 6 and 7 in \cite{zhang2018dynamic}, respectively, which in turn follow from Theorems \ref{finalm1} and \ref{finalm2} and Theorem 3 in \cite{zhang2018dynamic}.

\begin{theorem}
\label{thm:dynecsm-combined1}
\leavevmode
    In the setting where the cumulative saddle-point is provided to the MMFLH:
\begin{itemize}

    \item If $f_t$ are \ec, then upon using \ommns~ as the base learner in MMFLH, we have that the iterates $\{x_t,y_t\}_{t=1}^T$, generated by MMFLH satisfy
    $$
    \text{DSP-Reg}_T \leq O(d \max\{\log T,\sqrt{TV_T\log T}\})
    $$
    
    \item If $f_t$ are strongly convex-strongly concave, then upon using OGDA as the base learner in MMFLH, we have that the iterates $\{x_t,y_t\}_{t=1}^T$, generated by MMFLH satisfy
    $$
    \text{DSP-Reg}_T \leq O(\max\{\log T,\sqrt{TV_T\log T}\})
    $$
\end{itemize}
\end{theorem}
\begin{proof}
    Setting $K = \lceil T^\frac{1}{\gamma}\rceil$ in the bound on \sasp~derived in Theorems \ref{finalm1} and \ref{finalm2} gives us 
    $$\text{SASP-Reg}_T(\tau) \leq ( \frac{(2d + 1)(1+\gamma)+2}{\alpha} + 2d(1+\gamma)L_0D)\log(T),$$ 
    
    and,
    $$\text{SASP-Reg}_T(\tau)\leq \frac{L_0^2}{2\lambda}((1+\gamma) + (3\gamma + 7) \log  T )$$ 
    
    in the \ec~and strongly convex-strongly concave settings respectively.
    
    From Theorem 3 \cite{zhang2018dynamic} we have, 
    $$\text{DSP-Reg}_T \leq \min_{\mathcal{I}_1,...,\mathcal{I}_k}(\sum_{i=1}^k \text{SASP-Reg}_T(|\mathcal{I}|_i)+2|\mathcal{I}|_iV_T(i))$$
    
    where $V_T(i) = \sum_{t=s_i+1}^{q_i} \max_{(x,y)\in \cx\times\cy} |g_t'(x,y)-g_{t-1}'(x,y)|$ is the local functional variation.

    Setting $k = \frac{T}{\tau}$ in the upper bound of \dsp, we obtain
    \begin{align*}
        \text{DSP-Reg}_T
        &\le    
        \min_{1 \le \tau \le T}
        \sum_{i=1}^{k}
        \bigl(
        \mathrm{SASP\text{-}Regret}_T + 2\tau V_T(i)
        \bigr)
        \\
        &=
        \min_{1 \le \tau \le T}
        \left(
        \frac{\mathrm{SASP\text{-}Regret}_T(T,\tau)\, T}{\tau}
        + 2\tau \sum_{i=1}^{k} V_T(i)
        \right)
        \\
        &\le
        \min_{1 \le \tau \le T}
        \left(
        \frac{\mathrm{SASP\text{-}Regret}_T(T,\tau)\, T}{\tau}
        + 2\tau V_T
        \right).
    \end{align*} 

    This gives,
    $$\text{DSP-Reg}_T \leq \min_{1\leq \tau \leq T} \Big\{\pa{ \frac{(2d + 1)(1+\gamma)+2}{\alpha} + 2d(1+\gamma)L_0D}\frac{T\log T}{\tau}+2\tau V_T \Big\}$$

    and,

    \begin{align*}
        \text{DSP-Reg}_T &\leq \min_{1\leq \tau \leq T} \Big\{\frac{L_0^2}{2\lambda}\pa{(1+\gamma) + (3\gamma + 7) \log  T )}\frac{T}{\tau}+2\tau V_T\Big\}\\
        &\leq \min_{1\leq \tau \leq T} \Big\{\frac{L_0^2}{\lambda}\pa{\gamma + 5\gamma \log  T )}\frac{T}{\tau}+2\tau V_T\Big\}
    \end{align*}

    for the \ec~and strongly convex-strongly concave settings respectively.

    We now choose $\tau= \sqrt{\frac{T\log T}{V_T}}$ if $V_T\geq \log T/T$ else we choose $\tau=T$ and obtain,

    $$\text{DSP-Reg}_T \leq \max
\begin{cases}
\pa{ \frac{(2d + 1)(1+\gamma)+2}{\alpha} + 2d(1+\gamma)L_0D+2}\sqrt{TV_T\log T} & \text{if }  V_T\geq \frac{\log T}{T}, \\
\pa{ \frac{(2d + 1)(1+\gamma)+2}{\alpha} + 2d(1+\gamma)L_0D+2}\log T & \text{if } V_T\leq \frac{\log T}{T},
\end{cases}$$

and,
$$\text{DSP-Reg}_T \leq \max
\begin{cases}
 \frac{\gamma L_0^2}{\lambda}
\sqrt{\frac{T V_T}{\log T}}
+
\left(
\frac{5\gamma L_0^2}{\lambda} + 2
\right)
\sqrt{T V_T \log T} & \text{if }  V_T\geq \frac{\log T}{T}, \\
\frac{\gamma L_0^2}{\lambda}
\;+\;
\left(
\frac{5\gamma L_0^2}{\lambda} + 2
\right)\log T & \text{if } V_T\leq \frac{\log T}{T},
\end{cases}$$

for the \ec~and strongly convex-strongly concave settings respectively. This implies the statement of the theorem.
\end{proof}
\subsection{Relations between dynamic notions of regret}\label{app:rels}

\begin{theorem}
   For convex-concave functions $f_t$ we have that \dyne $\leq$ \dgap.
\end{theorem} 
\begin{proof}
    We follow the proof of proposition 11 \cite{zhang2022no}. We have,
    \begin{align*}
\sum_{t=1}^T f_t(x_t,y_t) - \sum_{t=1}^T \min_{x \in \cx} \max_{y \in \cy} f_t(x,y) 
&\leq \sum_{t=1}^T \max_{y \in \cy} f_t(x_t,y) - \sum_{t=1}^T \min_{x \in \cx} \max_{y \in \cy} f_t(x,y) \\
&= \sum_{t=1}^T \max_{y \in \cy} f_t(x_t,y) - \sum_{t=1}^T f_t(x_t^*,y_t^*) \quad \\
&\leq \sum_{t=1}^T \max_{y \in \cy} f_t(x_t,y) - \sum_{t=1}^T f_t(x_t^*,y_t) \\
&\leq \sum_{t=1}^T \max_{y \in \cy} f_t(x_t,y) - \sum_{t=1}^T \min_{x \in \cx} f_t(x,y_t).
\end{align*}

where $(x_t^*, y_t^*) \in  \cx\times\cy$ are the saddle points of the functions $f_t$.

\begin{align*}
\sum_{t=1}^T \min_{x \in \cx} \max_{y \in \cy} f_t(x,y) - \sum_{t=1}^T f_t(x_t,y_t) 
&\leq \sum_{t=1}^T f_t(x_t^*,y_t^*) - \sum_{t=1}^T \min_{x \in \cx} f_t(x,y_t) \\
&\leq \sum_{t=1}^T f_t(x_t,y_t^*) - \sum_{t=1}^T \min_{x \in \cx} f_t(x,y_t) \\
&\leq \sum_{t=1}^T \max_{y \in \cy} f_t(x_t,y) - \sum_{t=1}^T \min_{x \in \cx} f_t(x,y_t)
\end{align*}

combining the previous two we have the statement of the theorem.
\end{proof}

\begin{theorem}
\label{thm:dydyup1}
Consider convex-concave functions $\{f_t\}_{t=1}^T$. Furthermore, let the functions be represented as $f_t(x,y) = h_{1,t}(x)-h_{2,t}(y)$ where $h_{1,t}(x)$ and $h_{2,t}(y)$ are convex in $x$ and $y$ respectively. For any algorithm run on these functions in the online min-max setting, let the dynamic saddle-point regret and duality gap of the action-pairs $\{x_t,y_t\}_{t=1}^T$ be \dsp~and \dgap. Then we have \dsp $\geq$ \dgap.
\end{theorem}
\begin{proof}
    Define $z'=(x',y')$ to be the saddle point of the function $f = \sum_{t=1}^T f_t(x_t,y_t)$, i.e. $(x',y') = (\argmin_x \max_y \sum_{t=1}^T f_t(x,y),\argmax_y \min_x\sum_{t=1}^T f_t(x,y))$. Also recall $g'_t(x,y) = (f_t(x,y')-f_t(x',y ))$ , $g_t^*(z_t) = f(x_t, y^*(x_t)-f(x^*_t(y),y_t)$ where $y^*_t(x) = \argmin_{y''} f_t(x,y'')$ and $x^*_{t}(y) = \argmin_{x'} f_{t+1}(x',y)$. 
    
    We now consider,
    $$g'(z_t) = f(x_t,y')-f(x',y_t) = (h_1(x_t)-h_2(y'))-(h_1(x')-h_2(y_t))$$
    furthermore,
    $$g'(z^*_t(z')) = f(x^*_t(y'), y')-f(x',y^*_t(x')) = (h_1(x^*_t(y'))-h_2(y'))-(h_1(x')-h_2(y^*_t(x')))$$
    and finally,
    $$g^*( z_t) = f( x_t, (y^*(x_t))-f(x^*(y_t), y_t) = (h_1( x_t)-h_2((y^*(x_t)))-(h_1(x^*(y_t))-h_2( y_t)).$$
    Thus we have,
    $g'(z_t)-g'(z^*_t(z'))-g^*(z_t) = \big([h_1(x_t)-h_2(y')]-[h_1(x')-h_2(y_t)]\big)-\big([h_1(x^*_t(y'))-h_2(y')]-[h_1(x')-h_2(y^*_t(x'))]\big)-\big([h_1(x^*_t(y'))-h_2(y^*(x_t))]-[h_1(x^*(y_t))-h_2(y^*_t(x'))]\big)$
    which gives,
    $$g'(z_t)-g'(z^*_t(z'))-g^*(z_t) = (h_1(x_t)-h_1(x^*_t(y')))+(h_2(y_t)-h_2(y^*_t(x'))) \geq 0$$
    since $z^*_t(z') = z^*_t(z_t) ~\forall~t.$ 

    This is equivalent to \dsp$\geq$\dgap.
\end{proof} 
\subsection{Proofs for the online VI setting}\label{app:vi}

In this section we provide the precise definition of the lower regularity in definition \ref{def:vilr} and an algorithm that achieves a regret of $O(\log(T)$ on the online VI objective.

\begin{algorithm}[H]
\caption{Online VI optimization}
\label{alg:mainalgvi}
\KwIn{$z_0, u_0 \in \mathcal{X} \times \mathcal{Y}, split~s$}
\For{$t = 1$ \textbf{to} $T$}{
    Play $z_t$, observe cost $g_t(z_t)$\\
    Regularity update: $A_t = A_{t-1}+M_s(z_{t-1})$\\
    $z_{t} = z' ~s.t.~ \{\ang{F(z_{t-1})+A_{t-1}(z'-z_{t-1}),z'-z}\leq 0~\forall~z\}$
}
\end{algorithm}

Before we proceed to derive theorem \ref{thm:final}, we prove a supporting lemma regarding the algorithm.

\begin{lemma}\label{lem:inversebound}
    For iterates $z_t$ and $z_{t+1}$ and matrix $A_{t+1}$ for any iteration number $t$ we have, 
$$(z_{t} - z_{t+1})^\top F(z_{t}) - \frac{\gamma}{2}(z_{t+1} - z_{t})^\top A_{t} (z_{t+1} - z_{t}) \leq \frac{1}{2\gamma}F(z_{t})^\top (A_{t})^{-1} F(z_{t})$$
\end{lemma}
\begin{proof}

We want to complete the square for this expression. The quadratic term already has a negative coefficient, so we're looking for a form:
\begin{equation*}
-\frac{\gamma}{2}(z + b)^\top A_{t} (z + b) + c
\end{equation*}
for some vector $b$ and scalar $c$.
Expanding this form:
\begin{align*}
&-\frac{\gamma}{2}(z_{t} A_{t} z + z_{t} A_{t} b + b_{t-1} A_{t} z + b_{t-1} A_{t} b) + c \\
&= -\frac{\gamma}{2}z_{t} A_{t} z - \gamma z_{t} A_{t} b - \frac{\gamma}{2}b_{t-1} A_{t} b + c
\end{align*}
For this to match our original expression:
\begin{equation*}
-z_{t} F(z_{t}) - \frac{\gamma}{2}z_{t} A_{t} z
\end{equation*}
The linear terms must match, thus $-\gamma z_{t} A_{t} b = -z_{t} F(z_{t})
$. Therefore, $b = \frac{1}{\gamma}(A_{t})^{-1} F(z_{t})$.
Substituting back and solving for $c$ we obtain,
\begin{equation*}
c = \frac{\gamma}{2}b_{t-1} A_{t} b = \frac{1}{2\gamma}F(z_{t})^\top (A_{t})^{-1} F(z_{t})
\end{equation*}
Therefore our original expression equals:
\begin{equation*}
-\frac{\gamma}{2}\left(z + \frac{1}{\gamma}(A_{t})^{-1} F(z_{t})\right)^\top A_{t} \left(z + \frac{1}{\gamma}(A_{t})^{-1} F(z_{t})\right) + \frac{1}{2\gamma}F(z_{t})^\top (A_{t})^{-1} F(z_{t})
\end{equation*}

Since $A_{t}$ is positive definite, the first term is non-positive, so:
\begin{equation*}
(z_{t} - z_{t+1})^\top F(z_{t}) - \frac{\gamma}{2}(z_{t+1} - z_{t})^\top A_{t} (z_{t+1} - z_{t}) \leq \frac{1}{2\gamma}F(z_{t})^\top (A_{t})^{-1} F(z_{t})
\end{equation*}
\end{proof}

\begin{theorem}
\label{thm:final1}
The iterates $\{z_t\}_{t=1}^T$ generated by the low-rank Newton's method algorithm (Algorithm \ref{alg:mainalgvi}) obtains a regret of $O(\log T)$ on the online VI objective:
\begin{equation}\label{onlineVI}
    \sum_{t=1}^T\ang{F_t(z),z_t-z} \leq (\frac{1}{\alpha} + L_0D)(d\log T + 1)~~~~~~~~~\forall~z\in \cz,
\end{equation}
for operators satisfying the condition in Definition \ref{def:vilr} over a domain with bounded diameter. 
\end{theorem}
\begin{proof}

From the update, for any point $z=(x,y) \in \cx\times\cy$ we have,
$$\ang{F_t(z_t)+A_t(z_{t+1}-z_t),z_{t+1}-z}\leq 0$$
splitting we obtain,
$$(z - z_{t+1})^\top F_t(z_t) + \gamma(z - z_{t+1})^\top A_{t} (z_{t+1} - z_t) \geq 0$$
Observe that for any two vectors $a,b$ and matrix $M$ we have, $a^T M b = \frac{1}{2}[(a + b)^T M(a + b) - a^T M a - b^T M b]$. Applying this to the second term with $a = (z - z_{t+1})$, $b = (z_{t+1} - z_t)$ and $M=\gamma A_{t}$ we obtain,
\begin{align*}
    \gamma(z - z_{t+1})^\top A_{t} (z_{t+1} - z_t) &= 
\frac{\gamma}{2}[(z - z_t)^\top A_{t} (z - z_t) - (z - z_{t+1})^\top A_{t} (z - z_{t+1}) \\
&\quad- (z_{t+1} - z_t)^\top A_{t} (z_{t+1} - z_t)]
\end{align*}
Substituting this back into the original inequality:
\begin{align}
\frac{\gamma}{2} \Big[ &(z - z_t)^\top A_{t} (z - z_t) - (z - z_{t+1})^\top A_{t} (z - z_{t+1}) \notag \\
&- (z_{t+1} - z_t)^\top A_{t} (z_{t+1} - z_t) \Big] \geq \langle z_{t+1} - z, F_t(z_t) \rangle
\end{align}
Rearranging terms:
\begin{align}
\frac{\gamma}{2}(z - z_{t+1})^\top A_{t} (z - z_{t+1}) + \frac{\gamma}{2}(z - z_t)^\top A_{t} (z - z_t) 
&\leq (z - z_{t+1})^\top F_t(z_t) \notag \\
&\quad - \frac{\gamma}{2}(z_{t+1} - z_t)^\top A_{t} (z_{t+1} - z_t)
\end{align}
By splitting \((z - z_{t+1})^\top F_t(z_t)\) into \((z - z_t)^\top F_t(z_t) + (z_t - z_{t+1})^\top F_t(z_t)\), we get,
\begin{align*}
\frac{\gamma}{2}(z_{t+1} - z)^\top A_{t} (z_{t+1} - z) 
- \frac{\gamma}{2}(z_t - z)^\top A_t (z_t - z)&\leq (z_t - z_{t+1})^\top F_t(z_t)+(z - z_t)^\top F_t(z_t) \\
&\quad - \frac{\gamma}{2}(z_{t+1} - z_t)^\top A_{t} (z_{t+1} - z_t)\\
&\leq (z - z_t)^\top F_t(z_t) + \frac{1}{2\gamma} F_t(z_t)^\top (A_{t})^{-1} F_t(z_t) 
\end{align*}
where in the last step we use the results of Lemma \ref{lem:inversebound}.
Rearranging we obtain,
\begin{align}
    \frac{\gamma}{2}(z_{t+1} - z)^\top A_{t} (z_{t+1} - z) - \frac{\gamma}{2}(z_t - z)^\top A_{t-1} (z_t - z) \leq & \frac{\gamma}{2}
(z_t - z)^\top (A_{t} - A_{t-1})(z_t - z)\nonumber \\
&\quad\quad +(1/2\gamma)F_t(z_t)^\top (A_{t})^{-1} F_t(z_t)\nonumber\\
&\quad\quad +(z - z_t)^\top F_t(z_t)\nonumber
\end{align}
From the assumption we have,

$$\ang{z - z_t, F(z) - F_t(z_t)} \geq \frac{\gamma}{2}\|z-z_t\|_{M(z_t)}^2$$
combining the previous two we obtain,
\begin{align*}
    \frac{\gamma}{2}(z_t - z)^\top A_{t} (z_t - z) -\frac{\gamma}{2}(z_t - z)^\top A_{t-1} (z_t - z) \leq &
\frac{\gamma}{2}(z_t - z)^\top ((A_{t}-A_{t-1})-M(z_t))(z_t - z)\nonumber \\
&\quad\quad +\frac{1}{2\gamma}F_t(z_t)^\top (A_{t})^{-1} F_t(z_t)+(z- z_t)^\top F_t(z_t)
\end{align*}
summing over $k$ we have (terms on the LHS telescope),
\begin{align}
    \sum_{t=1}^T \ang{F_t(z_t),z_t-z} & \leq \sum_{t=1}^T \frac{1}{2\gamma}F_t(z_t)^\top (A_{t})^{-1} F_t(z_t) +\frac{\gamma}{2}\|z_{0}-z\|^2_{A_{0}} -\frac{\gamma}{2}\|z_{t}-z\|^2_{A_{T}}
\end{align}
We have that $\gamma \|z_0-z\|_{A_0}^2 \leq D^2 \varepsilon\gamma \leq \frac{1}{\gamma}$, where $D$ is the diameter and that $\|z_{t}-z\|_{A_{t}}$ is non-negative, which gives,
\begin{align}\label{eqn:onVI}
    \sum_{t=1}^T \ang{F_t(z),z_t-z} & \leq \sum_{t=1}^T \frac{1}{2\gamma}F_t(z_t)^\top (A_{t})^{-1} F_t(z_t) +\frac{1}{2\gamma}
\end{align}
Define $A_{s_i,t} = \sum_{r=1}^t F_{s_i,r}(z_{s_i,r})F_{s_i,r}(z_{s_i,r})^\top+\epsilon I_{s_i}$, we have, 
\begin{align*}
    \sum_{t=1}^T F_t(z_t)^\top (A_{t})^{-1} F_t(z_t) &=  \sum_{t=1}^T \sum_{i=1}^m(A_{t})^{-1}F_{s_i,t}(z_t)F_{s_i,t}(z_t)^\top\\
    &= \sum_{t=1}^T \sum_{i=1}^m (A_{s_i,t})^{-1}F_{s_i,t}(z_{s_i,t}) F_{s_i,t}(z_{s_i,t})^\top\\
    &= \sum_{t=1}^T \sum_{i=1}^m (A_{s_i,t})^{-1}(A_{s_i,t}-A_{s_i,t-1})\\
    &\leq \sum_{t=1}^T  \sum_{i=1}^m \log \frac{|A_{s_i,t}|}{|A_{s_i,t-1}|}\\
    &= \sum_{i=1}^{m}\log  \frac{|A_{s_i,T}|}{|A_{s_i,0}|}
\end{align*}
where the first equality follows since $a^\top b = Tr(ab^\top)$ (where $a$ and $b$ are two vectors and $Tr$ denotes the trace) and the last inequality follows from the fact that for two matrices $A$ and $B$, $A^{-1}(A-B) \leq \log \frac{|A|}{|B|})$ where $|.|$ denoted the determinant.

Since $A_{s_i,T} = \sum_{t=1}^T  F_{s_i,t}(z_{s_i,t})F_{s_i,t}(z_{s_i,t})^\top+\epsilon I_{s_i}$ and $\|F(z_{s_i,t})\|_2 \leq L_0$, the determinant of $A_{s_i,T}$ can be bounded by $|A_{s_i,T}| \leq (TL_0^2 + \varepsilon)^{s_i}$ (recall $A_{s_i,T} \in \mathbb{R}^{s_i\times s_i}$). Thus since, $\varepsilon = \frac{1}{\gamma^2D^2}$ and $\gamma = \frac{1}{2}\min\{\frac{1}{L_0D}, \alpha\}$, for $T > 4$,
\begin{align*}
\sum_{t=1}^T F_t(z_t)^\top A_T^{-1}F_t(z_t) 
&\leq \sum_{i=1}^m s_i\log(TL_0^2\gamma^2D^2 + 1) \\
&\leq d\log T.
\end{align*}

Plugging into Eq.~\eqref{eqn:onVI} and noting that $\frac{1}{\gamma} \leq 2(\frac{1}{\alpha}+L_0D)$ we obtain,
\begin{equation*}
 \sum_{t=1}^T \ang{F_t(z),z_t-z} \leq (\frac{1}{\alpha} + L_0D)(d\log T + 1)~\forall~z\in\cz
\end{equation*}
which implies the theorem.
\end{proof}

\section{Proofs and algorithm for time-varying games}\label{app:sec3n4}

\subsection{Proofs for AGDA and 2-sided PL}\label{app:agda-2PL}

\begin{restatable}{lemma}{localconLemma}
\label{lem:localcon}
Upon running the alternating gradient descent ascent algorithm with $\tau_1 = \frac{\mu_2}{18L_1^3}$ and $\tau_2 = \frac{1}{L_1}$ on a differentiable function satisfying the two-sided PL-inequality, we obtain:
$$\|F(x_t,y_t)\|^2 \leq \beta \rho^t c_0$$ 
where $\rho = (1-\frac{\mu_1\mu_2^2}{36L^3})$ and $\beta =  \max \Big\{
(4L_1^2\tau_1^2+1)\frac{2L^2}{\mu_1},
\;10(4L_1^2\tau_1^2+3)\frac{2L_1^2}{\mu_2}\Big\}$ all constants depend on $\mu_1, \mu_2$ and the Lipschitz constant $L_1$ of the operator $F$.
\end{restatable}
\begin{proof}
     Let $\{x_t,y_t\}_{t=1}^T$ be the iterates of the alternating gradient descent ascent algorithm. Then we have,
    \begin{align}
        \|F(x_t,y_t)\|^2 = \|\nabla_x f(x_t,y_t)\|^2+\|\nabla_y f(x_t,y_t)\|^2
    \end{align}
    Note that, $\|\nabla_y f(x_{t},y_t)\|^2 \leq (\|\nabla_y f(x_{t},y_t)-\nabla_y f(x_{t+1},y_t)\|+\|\nabla_y f(x_{t+1},y_t)\|)^2 \leq (L_1\|x_{t+1}-x_t\|^2+\|\nabla_y f(x_{t+1},y_t)\|)^2 \leq 2L_1^2\tau_1^2\|\nabla_x f(x_t,y_t)\|^2+2\|\nabla_y f(x_{t+1},y_t)\|^2$. We have from equation (32) (\cite{yang2020global}) that,
    $$ \|\nabla_y f(x_{t+1},y_t) \|^2 \leq  \frac{2L_1^2}{\mu_2}b_t+L_1^2\tau_1^2\|\nabla_x f(x_t,y_t)\|^2.$$
    Also from (33) (in \cite{yang2020global}) we have,
    $$\|\nabla_x f(x_t,y_t)\|^2 \leq \frac{2L^2}{\mu_1}a_t+\frac{2L_1^2}{\mu_2}b_t.$$
    where $L=L_1+\frac{L_1^2}{\mu_2}$, $a_t = \max_y f(x_t,y)-\min_x \max_y f(x,y)$ and $b_t = \max_y f(x_t,y)-f(x_t,y_t)$. 
    
    This gives,
\begin{align*}
    \|\nabla_y f(x_{t},y_t)\|^2 &\leq 2L_1^2\tau_1^2\|\nabla_x f(x_t,y_t)\|^2+ \frac{4L_1^2}{\mu_2}b_t+2L_1^2\tau_1^2\|\nabla_x f(x_t,y_t)\|^2\\
    &=\frac{4L_1^2}{\mu_2}b_t+4L_1^2\tau_1^2\|\nabla_x f(x_t,y_t)\|^2
\end{align*}

    Thus we have,
    \begin{align*}
        \|F_t(z_t)\|^2 &\leq \frac{4L_1^2}{\mu_2}b_t+(4L_1^2\tau_1^2+1)\|\nabla_x f(x_t,y_t)\|^2.\\
        &\leq \frac{4L_1^2}{\mu_2}b_t+(4L_1^2\tau_1^2+1)(\frac{2L^2}{\mu_1}a_t+\frac{2L_1^2}{\mu_2}b_t)\\
        &\leq (4L_1^2\tau_1^2+1)(\frac{2L^2}{\mu_1}a_t)+(4L_1^2\tau_1^2+3)\frac{2L_1^2}{\mu_2}b_t\\
         &\leq \max \left\{
(4L_1^2\tau_1^2+1)\frac{2L^2}{\mu_1},
\;10(4L_1^2\tau_1^2+3)\frac{2L_1^2}{\mu_2}\right\}(a_t+\frac{b_t}{10}).
    \end{align*}
    
    Furthermore, we know from Theorem 3.2 (in \cite{yang2020global}),
    $$(a_t+\frac{b_t}{10}) \leq  (1-\frac{\mu_1\mu_2^2}{36L^3})^t (a_0+\frac{b_0}{10})$$
    This gives us,
    $$\|F_t(z_t)\|^2 \leq  \max \left\{
(4L_1^2\tau_1^2+1)\frac{2L^2}{\mu_1},
\;10(4L_1^2\tau_1^2+3)\frac{2L_1^2}{\mu_2}\right\}(1-\frac{\mu_1\mu_2^2}{36L^3})^t (a_0+\frac{b_0}{10})$$ 
    we represent the constants $c_0 = (a_0+\frac{b_0}{10})$, $\rho = (1-\frac{\mu_1\mu_2^2}{36L^3})$, $\beta =  \max \Big\{
(4L_1^2\tau_1^2+1)\frac{2L^2}{\mu_1},
\;10(4L_1^2\tau_1^2$\\
$+3)\frac{2L_1^2}{\mu_2}\Big\}$ and obtain the statement of the lemma. 
\end{proof}

\begin{restatable}{lemma}{innloopLemma}
\label{lem:innloop}
Upon running the inner loop in Algorithm \ref{alg:mainalg1} for $K = \max\{1,\lceil\log_\rho \frac{g_t^*(z_t)}{4\beta c_0}\rceil\}$ many steps, we obtain:
$$g^*_{t}(x_{t+1},y_{t+1}) \leq  \frac{1}{4}g^*_t(x_t,y_t).$$
\end{restatable}
\begin{proof}
    For the second statement, since the function $f$ satisfies the 2-sided PL-inequality we have that $g^*(z_t) \leq \max\{\mu_1,\mu_2\} \|F(z_t)\|^2$. Let $\mu = \max\{\mu_1,\mu_2\}$, from Lemma \ref{lem:localcon} we get $g_t^*(z_{t+1}) = g_t^*(x_{t,K},y_{t,K}) \leq g_t^*(z_t) \frac{\beta (a_t+\frac{b_t}{10})}{g_t^*(z_t)}\rho^{K_t}$ where $a_t = \max_y f(x_t,y)-\min_x \max_y f(x,y)$, $b_t=\max_y f(x_t,y)- f(x_t,y_t)$ and after $K_t$ iterations, and we have $$g_t^*(z_{t+1}) \leq \frac{g_t^*(z_t)}{4},$$ for $K_t = \max\{1,\lceil\log_\rho \frac{g_t^*(z_t)}{4\beta (\max_y f(x_t,y)-\min_x\max_y f(x,y)+\frac{\max_y f(x_t,y)-f(x_t,y_t)}{10})))}\rceil\}$, where we plug in the values of $a_t$ and $b_t$.
\end{proof}

\agdaplTheorem*
\begin{proof}

    We have,
    \begin{align*}
        \sum_{t=1}^T g^*_t(z_t) &\stackrel{(a)}{\leq} g^*_1(z_a) + \sum_{t=2}^T \left(g^*_t(z_t)-g^*_{t-1}(z_t) + g^*_{t-1}(z_t) \right) \\
        &\stackrel{(b)}{=} g^*_1(z_a)-g^*_t(z_{t+1}) + U_T + \sum_{t=2}^{T} \left(g^*_{t-1}(z_t)\right) \\
        &\stackrel{(c)}{=} g^*_1(z_a)-g^*_t(z_{T+1}) + U_T + \sum_{t=1}^{T-1} \left(g^*_t(z_{t+1})\right) \\
        &\stackrel{(d)}{\leq} g^*_1(z_a)-g^*_t(z_{T+1}) + U_T + \frac{1}{4} \sum_{t=1}^{T-1} \left(g^*_t(z_t)\right)
    \end{align*}
    Where in $(a)$ we split the summation and add and subtract $g_{t-1}(z_t)$. In $(b)$ we define $U_T = \sum_{t=1}^T \max_{z} (g^*_t(z)-g^*_{t-1}(z))$. In $(c)$ we decrement the summation index and in $(d)$ we set $\gamma = 4$ and $K_t  = \max\{1,\lceil\log_\rho \frac{g_t^*(z_t)}{4\beta (\max_y f(x_t,y)-\min_x\max_y f(x,y)+\frac{\max_y f(x_t,y)-f(x_t,y_t)}{10})))}\rceil\}$. in lemma~\ref{lem:innloop}. Rearranging we obtain,
    $$\sum_{t=1}^T g^{*}_t(z_t)-\frac{1}{4}\sum_{t=1}^{T-1} g^*_t(z_t)  \leq U_T+(g^*_1(z_a)-g^*_T(z_{T+1}))$$
    Noting that $\frac{3}{4}\sum_{t=1}^T g^*_t(z_t) = \sum_{t=1}^T g^*_t(z_t)-\frac{1}{4}\sum_{t=1}^{T} g^*_t(z_t) \leq \sum_{t=1}^T g^*_t(z_t)-\frac{1}{4}\sum_{t=1}^{T-1} g^*_t(z_t)$ since $g_T(z_t)\geq 0$ and that $\frac{4}{3} \leq 2$ we obtain, $$\text{Dual-Gap}_T = \sum_{t=1}^T g^*_t(z_t)\leq 2U_T+2(g^*_1(z_a)-g^*_T(z_{T+1}))$$

    which is the first statement of the Theorem. For the second statement we have from (a) \dgap$\leq \sum_{t=1}^T \delta_t+2(g^*_1(z_a)-g^*_T(z_{T+1}))$ where,

\begin{align*}
\delta_t 
&= \big(\max_{y''} f_{t+1}(x_{t+1}, y'') - \min_{x''} f_{t+1}(x'', y_{t+1})\big)  \\
&\quad\quad + \big(\min_{x'} f_t(x', y_{t+1}) - \max_{y'} f_t(x_{t+1}, y')\big) \\[0.4em]
&= \max_{y''} \big( f_{t+1}(x_{t+1}, y'') - \max_{y'} f_t(x_{t+1}, y')\big)  \\
&\quad\quad + \big(\min_{x'} f_t(x', y_{t+1}) - \min_{x''} f_{t+1}(x'', y_{t+1})\big) \\[0.4em]
&= \big( f_{t+1}(x_{t+1}, y^*_{t+1}(x_{t+1}))
      - f_t(x_{t+1}, y^*_t(x_{t+1}))\big) \\
&\quad\quad + \big( f_t(x^*_t(y_{t+1}), y_{t+1})
      - f_{t+1}(x^*_t(y_{t+1}), y_{t+1})\big) \\[0.4em]
&= \big( f_{t+1}(x_{t+1}, y^*_{t+1}(x_{t+1}))
      - f_t(x_{t+1}, y^*_{t+1}(x_{t+1}))\big) \\
&\quad\quad + \big( f_t(x_{t+1}, y^*_{t+1}(x_{t+1}))
      - f_t(x_{t+1}, y^*_t(x_{t+1}))\big) \notag \\
&\quad\quad + \big( f_{t+1}(x^*_{t+1}(y_{t+1}), y_{t+1})
      - f_t(x^*_{t+1}(y_{t+1}), y_{t+1})\big) \\
&\quad\quad + \big( f_t(x^*_{t+1}(y_{t+1}), y_{t+1})
      - f_t(x^*_t(y_{t+1}), y_{t+1})\big) \\[0.4em]
&\leq 2 \max_{z}\big( f_{t+1}(z) - f_t(z)\big) \\
&\quad\quad + \big( f_t(x^*_{t+1}(y_{t+1}), y_{t+1})
      - f_t(x^*_t(y_{t+1}), y_{t+1})\big) \notag \\
&\quad\quad + \big( f_{t+1}(x_{t+1}, y^*_{t+1}(x_{t+1}))
      - f_t(x_{t+1}, y^*_{t+1}(x_{t+1}))\big) \\[0.4em]
&\leq 2 \max_{z}\big( f_{t+1}(z) - f_t(z)\big) \\
&\quad\quad + L \big\| y^*_{t+1}(x_{t+1}) - y^*_t(x_{t+1}) \big\| \\
&\quad\quad + L \big\| x^*_{t+1}(y_{t+1}) - x^*_t(y_{t+1}) \big\|.
\end{align*}

and, $y^*_{t+1}(x) = \argmax_{y'} f_{t+1}(x_t,y')$, $y^*_t(x) = \argmin_{y''} f_t(x,y'')$ and $x^*_{t+1}(y) = \argmin_{x''} f_{t+1}(x'',y)$, and $x^*_{t}(y) = \argmin_{x'} f_{t+1}(x',y)$. 
Moreover, we have:
\begin{align*}
    f_t(x_{t+1}, y^*_{t+1}(x_{t+1})) - f_t(x_{t+1}, y^*_t(x_{t+1})) 
    &= f_t(x_{t+1}, y^*_{t+1}(x_{t+1})) - f(x_t, y^*_t(x_t)) \\
    &\quad + \big( f(x_t, y^*_t(x_t)) - f_t(x_{t+1}, y^*_t(x_{t+1}))\big) \\
    &\leq L \|y^*_{t+1}(x_{t+1}) - y^*_t(x_t)\| 
    + L \|x_{t+1} - x_t\| + L' \|x_{t+1} - x_t\|.
\end{align*}
    where $L' = L+\frac{L^2}{\mu}$ from Lemma A.2 in \cite{yang2020global}. Thus we have by symmetry,
    \begin{align*}
        \delta_t &\leq 2\max_{z}(f_{t+1}(z)-f_{t}(z))+L\|y^*_{t+1}(x_{t+1})-y^*_t(x_t)\|+L\|x^*_{t+1}(y_{t+1})-x^*_t(y_t)\|\\
        &\quad\quad+(L'+L)\|z_{t+1}-z_t\|
    \end{align*}
    Thus overall we have,
    $$\text{D-Gap}_T \leq O(\min \{U_T,\{V_T+\sum_{t=1}^T \Delta_t+C_T\},\{V_T+C_T'\}\}).$$
\end{proof}

\subsection{Examples of 2-sided PL}\label{app:2sidedPL}

To motivate our setting, we recall some examples of the 2-sided PL condition based on LQR, as established in previous works~\citep{Cai2019OnTG, Zhang2019PolicyOP, Fazel2018GlobalCO}.

\textbf{Example 1.} The generative adversarial imitation learning for LQR can be formulated as 
\[
\min_K \max_\theta\, m(K, \theta),
\]
where $m$ is strongly concave in terms of $\theta$ and satisfies the PL~condition in terms of $K$ (see \cite{Cai2019OnTG} Lemma 5.2 for more details), thus satisfying the two-sided PL~condition. 

\textbf{Example 2.} In a zero-sum linear quadratic (LQ) game, the system dynamics are characterized by 
\[
x_{t+1} = A x_t + B u_t + C v_t,
\]
where $x_t$ is the system state and $u_t, v_t$ are the control inputs from two players. After parameterizing the policies of the two players by $u_t = -K x_t$ and $v_t = -L x_t$, the value function is
\[
C(K,L) = \mathbb{E}_{x_0 \sim \mathcal{D}} \!\left[ \sum_{t=0}^{\infty} 
\big( x_t^\top Q x_t + (K x_t)^\top R^u (K x_t) - (L x_t)^\top R^v (L x_t) \big) \right],
\]
where $\mathcal{D}$ is the distribution of the initial state $x_0$ and $R^u, R^v$ are positive definite matrices (see \cite{Zhang2019PolicyOP} for the formulation details). Player~1 (player~2) wants to minimize (maximize) $C(K,L)$, and the game is formulated as $\min_K \max_L C(K,L)$. Fixing $L$ (or $K$), $C(\cdot, L)$ (or $-C(K, \cdot)$) becomes an objective of an LQR problem, and therefore satisfies the PL~condition (Corollary~5 in \cite{Fazel2018GlobalCO}) when $\arg\min_K C(K,L)$ and $\arg\max_L C(K,L)$ are well-defined.

Note that Example 1 is nonconvex strongly-concave, while Example 2 is nonconvex-nonconcave.

\subsection{Impossibility result and relation between \dsp~and \dgap}\label{app:impossible}
Following the approach of \cite{cardoso2019competing}, we show that there exists an instance of a strongly monotone online min-max optimization problem for which the static NE regret does not allow us to bound the individual regrets. Furthermore, since this instance (by Theorem \ref{thm:smt2combined}) also satisfies the \ec~and two-sided PL inequality, we know that there is an instance in each of these classes of functions for which low static NE-Regret does not imply low individual regrets for each player.

\begin{theorem}
\label{thm:imp1}
 When the domain $\cx\times\cy$ is a two-dimensional simplex, there exists a sequence of functions, all either strongly convex or \ec, for which there is no sequence $\{x_t,y_t\}_{t=1}^T$ that satisfies $\textit{SNE-}Reg_T \leq o(T)$, $Reg_T^1 \leq o(T)$, and $Reg_T^2 \leq o(T)$ simultaneously.
\end{theorem}
\begin{proof}
    We prove this by contradiction, i.e. we assume that for all functions there exists an algorithm that achieves sub-linear bounds on the individual regrets and the static Nash equilibrium regret simultaneously functions $f_1(x,y) = \|x\|^2-\frac{1}{4}\|y\|^2+x^\top A_t^1y$ and $f_2(x,y) = \|x\|^2-\frac{1}{4}\|y\|^2+x^\top A_t^2y$ where,
\[
\begin{aligned}
A_t^1 &=
\begin{cases} 
\begin{bmatrix}
    1 & -1\\
    -1 & 1
\end{bmatrix}, & \text{if } 0< t \leq \frac{T}{2}, \\[1em]
\begin{bmatrix}
    0 & 0\\
    0 & 0
\end{bmatrix}, & \text{if } \frac{T}{2} < t < T.
\end{cases}
\quad
A_t^2 &=
\begin{cases} 
\begin{bmatrix}
    -1 & 1\\
    1 & -1
\end{bmatrix}, & \text{if } 0< t \leq \frac{T}{2}, \\[1em]
\begin{bmatrix}
    -1 & 1\\
    -1 & 1
\end{bmatrix}, & \text{if } \frac{T}{2} < t < T.
\end{cases}
\end{aligned}
\]

Since $\text{SNE-Reg}_T$ and $\text{Reg}^1_T$ are sublinear we have,
\begin{align*}
    |\sum_{t=1}^T f_t^j(x_t,y_t)-\min_x \max_y \sum_{t=1}^T  f_t^j(x,y)| \leq o(T)\\
    \max_y \sum_{t=1}^T f_t^j(x_t,y)-\sum_{t=1}^T  f_t^j(x_t,y_t) \leq o(T)
\end{align*}

opening the modulus with a negative sign and combining we obtain for both functions,
\begin{align}\label{eqn:master}
   \max_y \sum_{t=1}^T f_t^j(x_t,y)-\min_x \max_y \sum_{t=1}^T  f_t^j(x,y) = \max_y \sum_{t=1}^\frac{T}{2}  f_t^j(x_t,y) \leq o(T)~\forall~j\in\{1,2\}
\end{align}

where the second equality follows since we have $\min_x \max_y f_t^{j}(x,y) = 0~\forall~j\in\{1,2\}$. 

We assume that $x$ and $y$ are two-dimensional and lie in a simplex, and parameterize them as $x = [\rho, 1-\rho]$ and $y = [\beta, 1-\beta]$. For $j=1$ we have from \eqref{eqn:master},
\begin{align}\label{eqn:imp1}
    \max_\beta [\sum_{t=1}^T (\rho_t^2+(1-\rho_t)^2-\frac{1}{4}(\beta^2+(1-\beta)^2))+(\sum_{t=1}^\frac{T}{2}4\rho_t\beta -2\rho_t -2\beta+1)] \leq o(T)\nonumber\\
    \sum_{t=1}^T (\rho_t^2+(1-\rho_t)^2)-\frac{T}{4}-2\sum_{t=1}^\frac{T}{2} \rho_t+\frac{T}{2} \leq o(T) \nonumber\\
    \sum_{t=1}^T (\rho_t^2+(1-\rho_t)^2)-2\sum_{t=1}^\frac{T}{2} \rho_t+\frac{T}{4} \leq o(T)
\end{align}

where the second inequality follows by setting $\beta = 0$. Similarly for $j=2$ we have,

\begin{align}\label{eqn:imp2}
    \max_\beta [\sum_{t=1}^T (\rho_t^2+(1-\rho_t)^2-\frac{1}{4}(\beta^2+(1-\beta)^2))-(\sum_{t=1}^\frac{T}{2}4\rho_t\beta -2\rho_t -2\beta+1)-2\beta T+T] \leq o(T)\nonumber\\
    \sum_{t=1}^T (\rho_t^2+(1-\rho_t)^2)-\frac{T}{4}+2\sum_{t=1}^\frac{T}{2} \rho_t=\frac{T}{2} +T\leq o(T) \nonumber\\
    \sum_{t=1}^T (\rho_t^2+(1-\rho_t)^2)+2\sum_{t=1}^\frac{T}{2} \rho_t+\frac{T}{4} \leq o(T) 
\end{align}

summing the equations \eqref{eqn:imp1} and \eqref{eqn:imp2}, and using the fact that the arithmetic mean is greater than the geometric mean we obtain
\begin{align*}
    o(T) \geq 2\sum_{t=1}^T (\rho_t^2+(1-\rho_t)^2)+\frac{T}{2} \geq \frac{3T}{2} 
\end{align*}

which is a contradiction. From Theorem \ref{thm:smt2combined} we have that our functions are exponentially-concave and satisfy the two-sided PL-inequality and thus there is an instance for each of these classes where it is not possible to have sub-linear bounds on the individual regrets as well as the static Nash Equilibrium regret.
\end{proof}

\newpage

\end{document}